\documentclass{article}
\pdfoutput=1
\usepackage{iclr2027_conference,times}

\usepackage[utf8]{inputenc}
\usepackage[T1]{fontenc}
\usepackage{amsmath,amssymb,amsthm}
\usepackage{mathtools}
\usepackage{stmaryrd}
\usepackage{booktabs}
\usepackage{caption}
\usepackage{placeins}
\usepackage{graphicx}
\usepackage{wrapfig}
\usepackage{xurl}
\usepackage{hyperref}
\usepackage{cleveref}
\usepackage{enumitem}

\usepackage{tikz}
\usepackage{pgfplots}
\pgfplotsset{compat=1.18}
\usetikzlibrary{arrows.meta,positioning,shapes,calc}
\usepackage{algorithm}
\usepackage{algpseudocode}
\algdef{SE}[SWITCH]{Switch}{EndSwitch}[1]{\textbf{switch} #1 \textbf{of}}{\textbf{end switch}}
\algdef{SE}[CASE]{Case}{EndCase}[1]{\textbf{case} #1\textbf{:}}{}
\algtext*{EndCase}
\usepackage{xcolor}
\usepackage{colortbl}
\usepackage{mdframed}

\theoremstyle{plain}
\newtheorem{proposition}{Proposition}[section]
\newtheorem{corollary}[proposition]{Corollary}
\newtheorem{lemma}[proposition]{Lemma}
\theoremstyle{definition}
\newtheorem{definition}{Definition}[section]

\newcommand{\K}[1]{\mathbf{K}_{#1}}
\newcommand{\B}[1]{\mathbf{B}_{#1}}
\newcommand{\A}[1]{\mathbf{A}_{#1}}
\newcommand{\Cknow}[1]{\mathbf{C}_{#1}}
\newcommand{\announce}[1]{[{#1}]}

\newcommand{\Model}{\mathcal{M}}

\newcommand{\FIXME}[2]{}

\newcommand{\PSPACE}{\textsc{PSpace}}
\newcommand{\PiP}[1]{\Pi^{\mathrm{P}}_{#1}}
\newcommand{\bigO}{\mathcal{O}}

\newenvironment{conversation}{%
  \begin{mdframed}[linecolor=gray!40, linewidth=0.5pt, backgroundcolor=gray!5,
    innertopmargin=6pt, innerbottommargin=6pt, skipabove=6pt]
  \small
}{%
  \end{mdframed}
}
\newcommand{\turn}[2]{\textbf{#1:} #2\par\smallskip}

\newcommand{\repourl}{\url{https://github.com/memory009/Grounded-Continuation}}
\makeatletter
\newcommand{\plainfootnote}[1]{{\let\thefootnote\relax
  \long\def\@makefntext##1{\noindent ##1}\footnotetext{#1}}}
\makeatother

\title{Grounded Continuation: A Linear-Time Runtime Verifier for LLM Conversations}

\author{Qisong He, Jinwei Hu, Xinmiao Huang, Changshun Wu, Yi Dong, Xiaowei Huang \\
  School of Computer Science \& Informatics, University of Liverpool
}

\iclrfinalcopy % prints the author block; the header it sets is replaced below

\begin{document}

\maketitle
\lhead{Preprint}
% Unmarked footnote at the foot of page 1 with the repository link.
\plainfootnote{Code available at \repourl.}

\begin{abstract}
In a long conversation, an LLM may produce a fluent continuation that rests on premises the conversation has already abandoned. 
Context-manipulation attacks exploit precisely this weakness.
We address this problem with a runtime verifier.
An LLM Interpreter maps each utterance to one or more of eight epistemic operations, and then a symbolic engine applies these operations to a dependency map that records what every claim rests on and whether it still stands.
Based on the dependency map, checking whether a continuation is grounded then reduces to a walk over the map, linear in its size and requiring no LLM call.
Retraction propagates through the same map with a conflict-free guarantee and flags exactly the conclusions that lose support. 
Our experiments with five QA models demonstrate substantial improvements in QA accuracy at a low cost per query. 
On ReviseQA for belief revision and MemoryAgentBench's FactConsolidation split (MemAB-FC), the verifier outperforms a retrieval baseline and raises MemAB-FC single-hop accuracy from $0.46$--$0.95$ to $0.94$--$0.98$.
With the verifier, even the small 7B model overtakes unaided GPT-4o. 
When a GPT-4o Interpreter extracts every update from raw text rather than taking the benchmarks' structured updates, the verifier still outperforms the retrieval baseline.
Moreover, on MemAB-FC, QA prompts remain compact at $97$--$174$ tokens while the full-context baseline reaches $114.5$K. 
Retraction queries take less than a microsecond at $2000$ turns.

\end{abstract}

\section{Introduction}
\label{sec:intro}
Large language models produce continuations that can be locally fluent and pragmatically plausible yet \emph{ungrounded}: disconnected from the claims, observations, and revisions that the conversation has actually established \citep{sirdeshmukh2025multichallenge,wilie2024belief}. The deployment question is concrete. Given the next LLM output, can we check, \emph{at this turn} and within a tight latency budget, whether that output traces back to those prior commitments? Answering ``why did we reject approach A?'' or ``what assumptions does our current decision rest on?'' requires tracking dependencies and revisions, even when the relevant utterances remain inside the context window. What is missing is a maintained structure connecting the LLM's continuation to the conversation's prior commitments. \Cref{fig:motivation} illustrates the problem on a fact-revision task: retrieval-augmented generation (RAG) presents superseded and current facts together, whereas our verifier maintains their standing and follows links between active records to supply evidence to the QA model.

\begin{figure}[t]
\centering
\includegraphics[width=\linewidth]{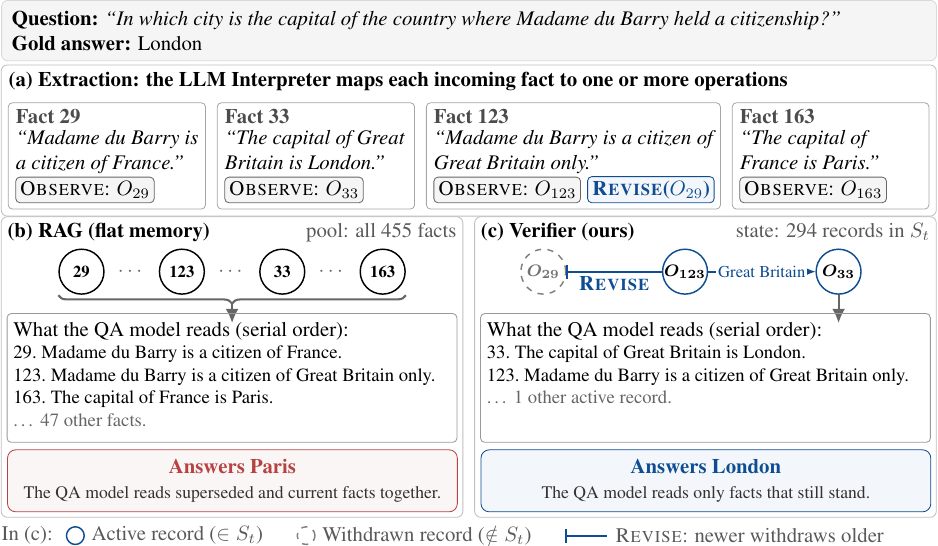}
\caption{\textbf{Maintaining standing changes the evidence available to QA.}
\textbf{(a)} The Interpreter maps incoming facts to \textsc{Observe} and \textsc{Revise} operations.
\textbf{(b)} RAG presents superseded and current facts together.
\textbf{(c)} The verifier follows the recorded values among the active records to the answer.}
\label{fig:motivation}
\end{figure}

This problem is empirically severe. \citet{laban2025lost} found a 39\% average performance drop in multi-turn versus single-turn settings across 200{,}000+ simulated conversations, with LLMs failing to revise incorrect assumptions even when later turns contradict them: ``when LLMs take a wrong turn in a conversation, they get lost and do not recover.'' \citet{shaikh2025rifts} show, on real human-LLM dialogues, that LLMs initiate clarification three times less often than human partners and that early grounding failures predict later breakdowns. The sharpest instance is continued reliance on a premise that the conversation has revised or superseded. The gap is also actively exploited: context-manipulation attacks against deployed agents~\citep{patlan2025fakememories,dong2025memoryinjection,hu2026lying} craft continuations that are locally consistent yet disconnected from the conversation's prior commitments, triggering unauthorised actions.

Retrieval-augmented attribution methods address \emph{external} grounding. A generated claim is accepted if it traces to a cited document \citep{gao2023alce, bohnet2023autoais}. We focus on \emph{interactional} grounding in long-horizon LLM conversations. A continuation is acceptable if it is supported by the claims, observations, hypotheses, and revisions made earlier in the same conversation. Mem0 consolidates and updates conversational facts \citep{chhikara2025mem0}, while Zep additionally tracks their temporal validity \citep{rasmussen2025zep}. Neither explicitly maintains the dependencies from premises to conclusions needed to check a continuation's support and propagate retractions. 
% A relevant conclusion may still rest on a withdrawn premise.

A runtime check must inspect either the generating network or the continuation's support in the conversation state. 
White-box verification of neural networks is established \citep{huang2017safety,huang2020survey} but faces severe complexity barriers for transformers. Even deciding emptiness is EXPSPACE-complete for fixed-precision unique-hard-attention transformers \citep{bergstrasser2026succinct}. We instead verify a continuation's support against the maintained state in linear time.

We implement this check with a runtime verifier. An LLM Interpreter maps each utterance to one or more of eight operations, and a symbolic engine applies them to a dependency map that records what each claim rests on and whether it still stands. At any turn the check asks whether a candidate continuation rests, through claims still in good standing, on the conversation's prior commitments, and answers in time linear in the map's size. Language understanding stays with the LLM, while dependency tracking stays with the engine, whose retraction guarantee holds by construction. The maintained state also supplies QA models with currently active evidence, connecting the grounding mechanism to downstream question answering.

% Our evaluation covers ReviseQA for belief revision \citep{helwe2025reviseqa}, MemoryAgentBench's fact-consolidation split \citep{hu2025memoryagentbench}, and ten memory abilities in BEAM's open-domain conversations \citep{tavakoli2026beam}. Across five QA models, gains on revision and consolidation persist under end-to-end extraction.

Our main contributions are as follows:
\begin{itemize}[leftmargin=*,itemsep=2pt,topsep=2pt]
\item \textbf{Runtime Verification of Interactional Grounding}: at each turn we check whether a candidate continuation traces back through a maintained dependency map to prior commitments.
\item \textbf{Better Performance and Versatility}: across three benchmarks, our verifier outperforms RAG on revision and consolidation, remains competitive on open-domain conversations, and lifts the weakest model (7B) above unaided GPT-4o on ReviseQA.
\item \textbf{Formal Guarantees}: on the maintained dependency map, we prove conflict-free updates, exact identification of a retraction's recorded dependents with unaffected arguments preserved, and linear-time grounding and retraction queries.
\item \textbf{Efficient Queries at Scale}: prompt tokens remain compact at $97$--$174$ against $6.0$K--$114.5$K for full context, and a retraction query runs in under a microsecond at $2000$ turns.
\end{itemize}

\section{The Runtime Verifier}
\label{sec:framework}

The verifier checks candidate continuations against a maintained dependency map. Its state combines an epistemic model of what is known, an argumentation framework of what attacks what and what currently stands, commitment records of who said what publicly, and the dependencies linking each claim to the premises it rests on. \Cref{def:eam} formalises these four components as $\mathcal{D}_t$.

\noindent\textbf{Preliminaries.} An argumentation framework is a set of arguments with an attack relation \citep{dung1995acceptability}. Two arguments \emph{conflict} if one attacks the other. A set of arguments is \emph{conflict-free} if it contains no conflicting pair, \emph{admissible} if it is conflict-free and attacks every attacker of its members, and a \emph{preferred extension} is a maximal admissible set. The epistemic side uses dynamic epistemic logic (DEL) over plausibility models. Its operators and semantics are recalled in \Cref{app:proofs}, and \Cref{app:phase1} calibrates the world-set machinery on the muddy-children puzzle (Phase~1).

\begin{definition}[Dependency map]
\label{def:eam}
Let $\mathit{Prop}_t$ be the conversation's propositions and $\mathit{Ags}$ its agents. A \emph{dependency map} at turn $t$ is a tuple $\mathcal{D}_t = (\Model_t, \mathit{AF}_t, \mathit{Cm}_t, \mathit{Dep}_t)$, where:
\begin{itemize}[itemsep=-2pt,leftmargin=*,topsep=0pt]
\item $\Model_t$ is an epistemic plausibility model over $\mathit{Prop}_t$ (\Cref{def:epm}; \citealt{baltag2008qualitative});
\item $\mathit{AF}_t = (\mathit{Args}_t, \mathit{Att}_t, \mathit{st}_t)$ is a labelled Dung argumentation framework with a finite set of arguments $\alpha\in\mathit{Args}_t$ carrying claims $\mathit{claim}(\alpha)\in\mathit{Prop}_t$, attack relation $\mathit{Att}_t \subseteq \mathit{Args}_t \times \mathit{Args}_t$, and standing labels $\mathit{st}_t : \mathit{Args}_t \to \{\mathsf{active}, \mathsf{weakened}, \mathsf{abandoned}, \mathsf{resolved}\}$;
\item $\mathit{Cm}_t : \mathit{Ags} \to \mathcal{P}(\mathit{Args}_t)$ records each agent's public commitments \citep{walton2008argumentation};
\item $\mathit{Dep}_t : \mathit{Args}_t \to \mathcal{P}(\mathit{Args}_t)$ maps each argument to its recorded supporting arguments, identified by their unique identifiers.
\end{itemize}
\end{definition}
We write $S_t := \{\alpha \in \mathit{Args}_t : \mathit{st}_t(\alpha) \in \{\mathsf{active}, \mathsf{resolved}\}\}$ for the \emph{active set} of arguments in good standing. Grounding and retraction queries read the current records without changing them. The formal foundation of these four components and the proofs of every statement in this section are in \Cref{app:proofs}. Given an associated record identifier, the dependency walk uses only $\mathit{Dep}_t$, argument membership, and standing labels.

\noindent\textbf{The runtime check.} The verifier tests a single property of a candidate continuation, \emph{groundedness} with respect to the current dependency map.
\begin{definition}[Grounded continuation]
\label{def:grounded}
Verification takes a candidate continuation $c$ asserting $\phi_c$ together with an associated record identifier $\alpha_c$ satisfying $\mathit{claim}(\alpha_c)=\phi_c$. The \emph{dependency walk} follows each link $\alpha\to\beta$ with $\beta\in\mathit{Dep}_t(\alpha)$, starting at $\alpha_c$. Write $\mathit{Dep}^{*}_t(\alpha_c)$ for all reachable arguments, including $\alpha_c$. The continuation is \emph{grounded} under this association iff every argument in this set belongs to $S_t$. Otherwise, including when no record is associated, it is \emph{ungrounded}.
\end{definition}
With the association fixed before the walk (\Cref{app:proofs}), the verifier returns:
\begin{equation*}
\mathsf{Verify}((c,\alpha_c), \mathcal{D}_t) = \begin{cases}
\langle \mathsf{grounded}, \mathit{Dep}^{*}_t(\alpha_c) \rangle & \text{if } \alpha_c \text{ is associated and its walk succeeds,} \\
\langle \mathsf{ungrounded}, \emptyset \rangle & \text{otherwise.}
\end{cases}
\end{equation*}
A grounded continuation carries the chain of upstream commitments it rests on. The walk stops at the first failing step. Ungrounded continuations are flagged for retry, retraction, or human review.

\begin{lemma}[Grounding correctness]
\label[lemma]{lem:grounding}
For a candidate associated with $\alpha_c\in\mathit{Args}_t$, the walk accepts iff $\mathit{Dep}^{*}_t(\alpha_c)\subseteq S_t$, and on acceptance returns exactly this dependency closure.
\end{lemma}

\noindent\textbf{Threat model.} Candidate continuations may arise from adversarial manipulation of the generating model's context or memories \citep{patlan2025fakememories,dong2025memoryinjection}, or from ordinary model drift. At verification time, we assume that the adversary may choose $c$ but cannot alter the conversation record supplied to the Interpreter or the maintained $\mathcal{D}_t$. Verification reads this state without updating it. For the associated record, \Cref{lem:grounding} ensures that an accepted continuation has a support chain consisting entirely of arguments in good standing in $\mathcal{D}_t$:
\begin{equation*}
\mathsf{Verify}((c,\alpha_c), \mathcal{D}_t) = \langle \mathsf{grounded}, \mathit{Dep}^{*}_t(\alpha_c) \rangle \;\Longrightarrow\; \mathit{Dep}^{*}_t(\alpha_c) \subseteq S_t.
\end{equation*}

\noindent\textbf{Retraction queries.} The engine also tracks how premise retractions affect existing arguments. It follows dependency links in reverse from the withdrawn premise's recorded argument, distinguishing separate assertions of the same proposition.
\begin{definition}[Affected set]
\label{def:affected}
Let $\beta \in \mathit{Args}_t$ identify the premise record queried for retraction. The one-step and iterated affected sets are
\begin{align*}
\mathit{Affected}(\beta) &:= \{\alpha \in \mathit{Args}_t : \beta \in \mathit{Dep}_t(\alpha)\}, \\
X_0 &:= \mathit{Affected}(\beta), \\
X_{k+1} &:= X_k \cup \{\alpha \in \mathit{Args}_t : \mathit{Dep}_t(\alpha) \cap X_k \neq \varnothing\} \quad (k \geq 0), \\
\mathit{Affected}^{*}(\beta) &:= \textstyle\bigcup_{k \geq 0} X_k .
\end{align*}
\end{definition}
$\mathit{Affected}^{*}(\beta)$ reports the recorded dependents for actual or counterfactual retraction, without changing their standing. The sequence $X_k$ stabilises within $|\mathit{Args}_t|$ rounds. Both sets include records already out of standing. Their intersection with $S_t$ identifies the currently active dependents.

\noindent\textbf{Updating $\mathcal{D}_t$.}\label{subsec:operations} The structure the walk traverses is not a list of mentioned propositions. It is populated by hypotheses generated through abduction (\Cref{def:abduction}) over an awareness structure (\Cref{def:awareness}) the conversation expands, and its standing labels, attacks, and dependencies are updated jointly each turn. An LLM Interpreter maps each utterance to one or more of eight operations and extracts their arguments. The engine applies the resulting operation instances in a fixed order, using the per-operation rule in \Cref{alg:apply}, to produce the per-turn update $\mathcal{D}_t \mapsto \mathcal{D}_{t+1}$ over $\Model$, $\mathit{AF}$, $\mathit{Cm}$, and $\mathit{Dep}$.

\noindent The eight operations split into three roles. \textsc{Observe} and \textsc{Resolve} commit content. \textsc{Hypothesize}, \textsc{Support}, \textsc{Undermine} and \textsc{Revise} adjust plausibility or attacks without erasing provenance. \textsc{Expand-Awareness} and \textsc{Question} expand or query without committing claims. The taxonomy collects the update primitives of the four components. Multiple operation instances from one utterance are composed into the same per-turn state transition.

In the table below, $\psi$ denotes an observed formula, $\gamma$ a hypothesis, $p$ a proposition, $\chi$ a query, and $\mathcal{A}_i$ the awareness set of agent $i \in \mathit{Ags}$ (\Cref{def:awareness}).

\begin{center}
\begin{minipage}{\linewidth}
\centering
\small
\textbf{The eight operations and their DEL realisations}\\[2pt]
\begin{tabular}{@{}p{1.8cm}p{7.2cm}p{4.3cm}@{}}
\toprule
\textbf{Operation} & \textbf{DEL realisation} & \textbf{Example} \\
\midrule
\textsc{Observe} & $\announce{!\psi}$ (hard public announcement) & ``I see 401 errors.'' \\[2pt]
\textsc{Hypothesize} & $\announce{\Uparrow\gamma}$ (soft lexicographic upgrade) & ``Could be a retry loop.'' \\[2pt]
\textsc{Support} & $\announce{\uparrow\gamma}$, or $\announce{\Uparrow\gamma}$ if $\gamma$-specific evidence & ``Timing is consistent.'' \\[2pt]
\textsc{Undermine} & family of plausibility downgrades of $\gamma$ & ``Wrong error code, 401, not 503.'' \\[2pt]
\textsc{Revise} & $\announce{!\lnot\gamma}$ or $\mathit{Att}$ edge addition & ``It's reversed: A causes B.'' \\[2pt]
\textsc{Expand-Awareness} & add $p$ to the awareness set $\mathcal{A}_i$; world-set refinement on $p$ & ``The DB alert is really Redis.'' \\[2pt]
\textsc{Resolve} & consensual $\announce{!\gamma}$, or decision recording dissent & ``Confirmed: bug is line~42.'' \\[2pt]
\textsc{Question} & no DEL update; $\chi$ queued as an open abductive problem & ``Why is traffic 3$\times$ normal?'' \\
\bottomrule
\end{tabular}
\end{minipage}
\end{center}

\noindent\textbf{Soundness of selective retraction.} The operations maintain one invariant on standing, and the retraction guarantee rests on it alone. For $X \subseteq \mathit{Args}_t$, write $\mathit{AF}_t \!\restriction\! X$ for $\mathit{AF}_t$ with the arguments in $X$ removed. We use this restriction to analyse retraction under the semantics of \citet{dung1995acceptability}. The maintained map retains the withdrawn records.

\begin{lemma}[Standing invariant]
\label[lemma]{lem:standing}
Under the operation preconditions of \Cref{tab:ops-full}, no argument of $S_t$ is the target of an attack in $\mathit{Att}_t$. In particular, $S_t$ is conflict-free in $\mathit{AF}_t$ at every turn.
\end{lemma}

\noindent Attacks enter only through \textsc{Undermine} and \textsc{Revise}, which leave their target outside $S_t$, and \textsc{Resolve} applies only to arguments in $S_t$ with no recorded incoming attack. Weakened and abandoned arguments therefore stay outside $S_t$. A renewed assertion of their claim enters as a new argument with its own support references.

\begin{proposition}[Conflict-free selective retraction]
\label{prop:dep-sound}
Let $S_t$ be the active set after a premise record $\beta$ is withdrawn, so $\beta \notin S_t$. Write $R := \{\beta\} \cup \mathit{Affected}(\beta)$. Then $S' := S_t \setminus R$ is conflict-free in $\mathit{AF}_t \!\restriction\! R$, some preferred extension of this restriction contains $S'$, and every such extension preserves every $\alpha \in S_t$ with $\beta \notin \mathit{Dep}_t(\alpha)$.
\end{proposition}

\begin{corollary}[Exact iterated retraction]
\label[corollary]{cor:dep-closure}
For any $\beta \in \mathit{Args}_t$, $\mathit{Affected}^{*}(\beta)$ contains exactly the arguments with a nonempty recorded dependency path to $\beta$. Under the hypotheses of \Cref{prop:dep-sound}, writing $R^{*} := \{\beta\} \cup \mathit{Affected}^{*}(\beta)$, $S_t \setminus R^{*}$ is conflict-free in $\mathit{AF}_t \!\restriction\! R^{*}$, and every $\alpha \in S_t$ with no such path is preserved.
\end{corollary}

\noindent\textbf{Per-turn tractability.} Given the associated record, grounding and retraction reduce to reachability over $\mathit{Dep}_t$, computable in time linear in $|\mathit{Args}_t|+|\mathit{Dep}_t|$ with no LLM call. Standing is maintained by the update operations (\Cref{lem:standing}), so the walk reads labels without recomputing argumentation extensions. For state updates, we analyse tractability under three structural restrictions satisfied by our scenarios: (A1) at most 12 agents; (A2) an acyclic attack graph, inherited from IBIS-style deliberation \citep{kunz1970ibis} in which each new argument supports, attacks, or refines an existing one rather than cycling back; (A3) an explicit finite world set $W$ for $\Model_t$ (\Cref{def:epm}). The general-case complexity is discussed in \Cref{app:proofs}.

\begin{proposition}[Tractability under structural restrictions]
\label{prop:complexity}
Under (A1)--(A3), per-turn computation is polynomial in $|W|$ and $|\mathit{Args}_t|$: $\bigO(|W|)$ per hard announcement, $\bigO(|W| \log |W|)$ per soft upgrade, $\bigO(1)$ per standing update, and $\bigO(|\mathit{Args}_t| + |\mathit{Dep}_t|)$ per $\mathsf{Verify}$ walk or $\mathit{Affected}^{*}$ query over a maintained reverse dependency index.
\end{proposition}

\section{Architecture and Validation Scenarios}
\label{sec:phases}\label{sec:architecture}

The verifier runs as a loop between an LLM Interpreter, a symbolic engine, and a context renderer (\Cref{fig:architecture}), validated on two authored multi-agent scenarios, Phases~2 and~3. Throughout, $o_i$ denotes an observation, $h_i$ a hypothesis, $a_i$ an assumption a decision rests on, and T$n$ the $n$-th turn.

\begin{figure}[t]
\centering
\definecolor{symFill}{HTML}{E8F0F8}\definecolor{symEdge}{HTML}{0F4D92}%
\definecolor{llmFill}{HTML}{FBF0DB}\definecolor{llmEdge}{HTML}{DC8A06}%
\definecolor{txtFill}{HTML}{F5F5F5}\definecolor{txtEdge}{HTML}{8C8C8C}%
\begin{tikzpicture}[
  box/.style={rectangle, draw, line width=0.8pt, rounded corners, minimum width=2.5cm, minimum height=0.85cm, align=center, font=\small},
  arr/.style={-{Stealth[length=3mm]}, thick},
  exlbl/.style={font=\scriptsize, text=black!70, align=center, inner sep=2pt},
  node distance=2.3cm
]
\node[box, fill=txtFill, draw=txtEdge] (conv) {Conversation\\(natural language)};
\node[box, fill=llmFill, draw=llmEdge, right=of conv] (llm) {LLM\\Interpreter};
\node[box, fill=symFill, draw=symEdge, right=of llm] (sym) {Symbolic\\Engine};
\node[box, fill=symFill, draw=symEdge, below=1.1cm of sym] (render) {Context\\Renderer};
\node[box, fill=txtFill, draw=txtEdge, below=1.1cm of llm] (ctx) {LLM Context\\Window};
\draw[arr] (conv) -- node[exlbl, above, text width=2.2cm] {\itshape ``it's reversed: \\ Auth$\to$Redis''} (llm);
\draw[arr] (llm) -- node[exlbl, above, text width=2.2cm] {\textsc{Hypothesize} \\ +\,\textsc{Revise}} (sym);
\draw[arr] (sym) -- node[exlbl, right, text width=2.6cm, anchor=west, xshift=2pt] {add $h_4$;\\ flip causal arrow;\\ revise plausibility} (render);
\draw[arr] (render) -- node[exlbl, below, text width=3cm] {\itshape ``$h_4$ now active''} (ctx);
\draw[arr, dashed] (ctx) -- node[font=\scriptsize, left] {informs} (llm);
\end{tikzpicture}
\caption{Hybrid architecture, illustrated with Phase~2 T12 (an engineer reverses the causal direction). Solid: utterance flow, classified operations, engine update, summary. Dashed: the rendered state fed back into the LLM's context.}
\label{fig:architecture}
\end{figure}

\noindent\textbf{LLM Interpreter and Symbolic Engine.} For each utterance $u_t$, the Interpreter performs (1)~\emph{operation classification} into the eight operations of \Cref{subsec:operations} and (2)~\emph{proposition extraction} of the relevant content and record identifiers. The engine applies the extracted operations in a fixed order to update standing and dependencies. \Cref{tab:ops-full} gives the operations' preconditions, and grounding and retraction queries then inspect the resulting state.

\noindent\textbf{Context Renderer.} After each engine update, the Renderer summarises the changed parts of $\mathcal{D}_t$ in natural language and writes the summary into the LLM's context window (dashed arrow in \Cref{fig:architecture}). The next classification step can then read the \emph{rendered state} of $\mathcal{D}_t$ alongside the conversation, as in the state-augmented condition of \Cref{tab:cross-model}. The Renderer's output is also what users see when the verifier flags an ungrounded continuation: ``decision $\alpha$ no longer grounded; affected by retraction of $a_3$; depends on $a_1, a_2$.''

\noindent\textbf{Phase 2: Naturalistic debugging.}
Three engineers debug a cascading failure (\Cref{app:phase2}). The scenario exercises hypothesis lifecycle tracking and counterfactual retraction: when an observation is invalidated, which conclusions lose grounding? Ground truth: token bug $\to$ retry storm $\to$ Redis exhaustion $\to$ rate-limit bypass $\to$ Stripe 429s, plus a monitoring miscategorisation. The conversation generates four hypotheses with non-trivial lifecycles, two of them decisive: $h_2$ (Redis$\to$Auth) is undermined at T7 by error-code evidence and abandoned at T9, and $h_4$ at T12 reverses the causal direction, subsuming $h_1$ and $h_3$ at T13. The verifier is queried with four retractions ($\mathit{Affected}^{*}(o_8), \mathit{Affected}^{*}(o_9), \mathit{Affected}^{*}(o_6)$, and the three-hop $\mathit{Affected}^{*}(o_5)$ via reverse traversal $o_5 \to h_3 \to h_4 \to h_1$). Correctness depends on whether the dependency map records \emph{why} $h_2$ was abandoned and preserves the provenance of $h_4$'s causal reversal: both are recorded in the map, whereas flat history must be replayed to recover them.

\noindent\textbf{Phase 3: Open-ended deliberation.}
A team chooses a real-time collaboration architecture (\Cref{app:phase3}). The scenario exercises commitment tracking and selective revision under assumption change: when an upstream assumption flips, which decisions need re-grounding and which remain intact? The outcome is a \emph{decision with dissent}: Alice commits to Yjs server-relayed under three assumptions ($a_1$ docs stay short; $a_2$ editing remains burst-mode; $a_3$ Q2 long-running documents unconfirmed). Bob records dissent. When $a_3$ later flips, $\mathit{Affected}^{*}(a_3)$ must flag exactly the Yjs decision while leaving the access-control resolution intact: the selective re-grounding capability that motivates the verifier.

\section{Experiments}
\label{sec:system}\label{sec:experiment}

Our experiments address the following four questions. Does maintaining standing improve QA performance, and on which tasks? Does this improvement persist when every update is interpreted from raw text? What does dependency closure contribute beyond a one-step retraction query? What does maintaining and querying the state cost?

\subsection{Experimental setup}
\label{sec:setup}

\noindent\textbf{Benchmarks.} Three third-party benchmarks carry the main evaluation (protocols in \Cref{app:setup}). \emph{ReviseQA} \citep{helwe2025reviseqa} poses a conclusion to re-evaluate after each of seven sequential context edits, and we run all $930$ verified scenarios. \emph{MemAB-FC}, the FactConsolidation split of MemoryAgentBench \citep{hu2025memoryagentbench}, streams numbered facts of which $35$--$40\%$ are later replaced and asks for \emph{current} values at four stream lengths from 6K to 262K tokens. \emph{BEAM} \citep{tavakoli2026beam} holds, in its 100K split, $20$ long open-domain conversations with $20$ probing questions each, two per memory ability from information extraction to knowledge update.

\noindent\textbf{Arms.} \emph{LLM-only} answers from the raw context. \emph{RAG} \citep{lewis2020rag} retrieves from the conversation itself, following each benchmark's own design where one exists. \emph{Verifier (native)} feeds the engine the benchmark's structured updates, and the QA model reads the recorded premises. \emph{Verifier (end-to-end)} has a GPT-4o Interpreter produce every update from the raw text, and is the only verifier arm on BEAM, which supplies no structured updates. The two verifier arms separate the engine's guarantees (\Cref{lem:standing,prop:dep-sound}) from the Interpreter's faithfulness in recovering $\mathcal{D}_t$. On ReviseQA the verifier reads the premises the conclusion rests on, and RAG retrieves top-$22$ statements from the post-edit context. On MemAB-FC RAG retrieves top-$50$ facts, and Verifier (native) follows subject--object links from the entities the question names. On BEAM RAG retrieves turn pairs, and the verifier reads source turns with a retraction list, $3.3$K and $3.7$K tokens on average, respectively (\Cref{app:setup,app:beam}).

\noindent\textbf{Models and metrics.} The QA model varies over Qwen2.5-7B-Instruct and Qwen2.5-14B-Instruct \citep{yang2024qwen25}, Gemma-3-12B-it \citep{gemmateam2025gemma3}, GPT-4o-mini \citep{openai2024gpt4omini} and GPT-4o \citep{openai2024gpt4o}, with the Interpreter held at GPT-4o. The classification experiment (\Cref{tab:cross-model}) varies the Interpreter instead, over Claude Sonnet 4 \citep{anthropic2025claude4}, GPT-4o, Qwen2.5-7B-Instruct and Llama-3.1-8B-Instruct \citep{grattafiori2024llama3}. ReviseQA reports its own accuracy@$k$, the fraction of scenarios whose first $k$ edits are all answered correctly, with $k=2$ and $7$ the benchmark's easy and hard levels. MemAB-FC reports single-hop and multi-hop accuracy pooled over the four stream lengths. Classification reports multi-label F1 and exact match, overall and on each scenario's \emph{key shifts}, the turns at which standing or awareness changes. BEAM scores each answer with an LLM judge \citep{zheng2023judging} against the benchmark's rubric.

\subsection{Downstream effectiveness}
\label{sec:results}

\begin{table}[t]
\centering
\caption{Five QA models under three arms. The verifier arm is native, and \Cref{sec:e2e} compares it with the end-to-end arm. MemAB-FC pools the four stream lengths. Per-depth and per-length results and paired tests are in \Cref{app:reviseqa-models,app:cr}.}
\label{tab:reviseqa-multimodel}
\small
\resizebox{\ifdim\width>\textwidth\textwidth\else\width\fi}{!}{
\begin{tabular}{l ccc ccc}
\toprule
& \multicolumn{3}{c}{ReviseQA (accuracy@2\,/\,@7)} & \multicolumn{3}{c}{MemAB-FC (single-hop\,/\,multi-hop)} \\
\cmidrule(lr){2-4} \cmidrule(lr){5-7}
QA model & LLM-only & RAG & Verifier & LLM-only & RAG & Verifier \\
\midrule
Qwen2.5-7B  & 0.127\,/\,0.008 & 0.356\,/\,0.149 & \textbf{0.544}\,/\,\textbf{0.202} & 0.383\,/\,0.055 & 0.463\,/\,0.040 & \textbf{0.975}\,/\,\textbf{0.642} \\
Qwen2.5-14B & 0.222\,/\,0.010 & 0.388\,/\,0.168 & \textbf{0.574}\,/\,\textbf{0.230} & 0.443\,/\,0.048 & 0.760\,/\,0.083 & \textbf{0.943}\,/\,\textbf{0.713} \\
Gemma-3-12B & 0.374\,/\,0.045 & 0.512\,/\,0.283 & \textbf{0.612}\,/\,\textbf{0.312} & 0.398\,/\,0.048 & 0.917\,/\,0.083 & \textbf{0.983}\,/\,\textbf{0.733} \\
GPT-4o-mini & 0.277\,/\,0.015 & 0.416\,/\,0.123 & \textbf{0.511}\,/\,\textbf{0.157} & 0.395\,/\,0.052 & 0.708\,/\,0.065 & \textbf{0.970}\,/\,\textbf{0.715} \\
GPT-4o      & 0.476\,/\,0.161 & 0.508\,/\,0.242 & \textbf{0.628}\,/\,\textbf{0.348} & 0.400\,/\,0.060$^{\dagger}$ & 0.950\,/\,0.130 & \textbf{0.978}\,/\,\textbf{0.738} \\
\bottomrule
\end{tabular}
}
\par\smallskip
\parbox{\textwidth}{\footnotesize $^{\dagger}$\,GPT-4o LLM-only was run on the 6K stream only.}
\end{table}

\begin{figure}[t]
\centering
\includegraphics[width=\linewidth]{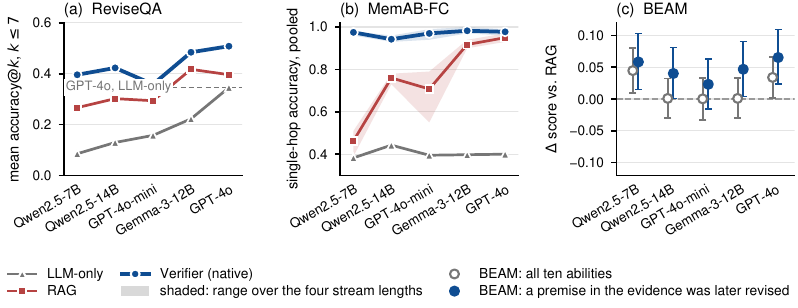}
\caption{QA models ordered by LLM-only accuracy on ReviseQA. \textbf{(a)} Mean accuracy@$k$ for $k\le7$. Dashed: GPT-4o, LLM-only. \textbf{(b)} MemAB-FC single-hop accuracy, pooled over stream lengths. Shading: range over lengths. \textbf{(c)} BEAM: Verifier (end-to-end) score relative to RAG, on all ten abilities and on questions with a later-revised premise. Error bars: $95\%$ bootstrap intervals.}
\label{fig:capability}
\end{figure}

\noindent\textbf{Verifier versus retrieval.} On ReviseQA the native verifier leads RAG at every revision depth on all five QA models, significantly on four of them at every depth and on Gemma-3-12B through $k\le6$ (exact McNemar, $p<0.05$; \Cref{tab:reviseqa-multimodel}, per-depth curves in \Cref{app:reviseqa-models}). On MemAB-FC it lifts single-hop accuracy from RAG's $0.46$--$0.95$ to $0.94$--$0.98$ and multi-hop from $0.04$--$0.13$ to $0.64$--$0.74$. On the $336$ of $400$ multi-hop questions for which the active records contain a chain linking the question to the gold answer, it reaches $0.74$--$0.85$. For $61$ of the other $64$, the chain to the gold answer passes through a superseded record (\Cref{app:mh-errors}). The gain is nearly constant across QA models, $0.60$--$0.65$ on multi-hop. Under retrieval the strongest model answers three times as many multi-hop questions as the weakest, and under the verifier all five fall within $0.10$ of each other. Our RAG baseline is stronger than the benchmark's own retrieval agent (\Cref{app:cr}).

\noindent\textbf{QA model capability.} Every QA model improves under the verifier, and the weakest readers gain most: mean accuracy@$k$ for $k\le7$ rises from $0.085$ to $0.396$ on Qwen2.5-7B and from $0.345$ to $0.508$ on GPT-4o (\Cref{fig:capability}a). Qwen2.5-7B under the verifier, at $0.396$, passes GPT-4o reading the raw context, at $0.345$, and exceeds GPT-4o with RAG on MemAB-FC single-hop, $0.975$ versus $0.950$ (\Cref{tab:reviseqa-multimodel}). Gains this uneven reorder the models: on ReviseQA, GPT-4o-mini leads Qwen2.5-7B unaided and trails it under the verifier.

\noindent\textbf{Open-domain conversations.} BEAM tests the verifier on long open-domain conversations. We label from the dependency map whether a question's evidence holds a premise the conversation later revised (\Cref{app:beam}). Pooled over the five QA models, the end-to-end verifier scores $0.347$ against RAG's $0.353$ on the $780$ question-model pairs where every premise still stands, and $0.249$ against $0.202$ on the $995$ where one was revised (\Cref{fig:capability}c). Paired Wilcoxon tests put the verifier ahead on three QA models where a premise was revised, and detect no significant difference on any of the five where every premise stands.

\subsection{End-to-end interpretation}
\label{sec:e2e}\label{sec:analysis}

\begin{table}[h]
\centering
\caption{Phase~2 classification under the Minimal (Min.), Definitions (Def.)\ and State-augmented (State-aug.)\ prompt conditions. Each cell is F1\,/\,Exact match in percent. Phase~3 in \Cref{tab:phase3-results}.}
\label{tab:cross-model}
\small
\resizebox{\ifdim\width>\textwidth\textwidth\else\width\fi}{!}{
\begin{tabular}{lcccccc}
\toprule
& \multicolumn{3}{c}{\textbf{Overall (13 turns)}} & \multicolumn{3}{c}{\textbf{Key shifts (4 turns)}} \\
\cmidrule(lr){2-4} \cmidrule(lr){5-7}
Interpreter & Min. & Def. & State-aug. & Min. & Def. & State-aug. \\
\midrule
Claude Sonnet 4 & 0.66\,/\,11 & 0.85\,/\,54 & \textbf{0.91}\,/\,\textbf{63} & 0.50\,/\,5 & 0.74\,/\,25 & \textbf{0.83}\,/\,30 \\
GPT-4o & 0.69\,/\,31 & 0.74\,/\,40 & 0.81\,/\,52 & 0.43\,/\,0 & 0.50\,/\,10 & 0.74\,/\,\textbf{50} \\
Qwen2.5-7B & 0.54\,/\,29 & 0.57\,/\,26 & 0.59\,/\,29 & 0.39\,/\,10 & 0.47\,/\,0 & 0.47\,/\,0 \\
Llama-3.1-8B & 0.54\,/\,12 & 0.55\,/\,20 & 0.65\,/\,32 & 0.37\,/\,0 & 0.47\,/\,20 & 0.56\,/\,15 \\
\bottomrule
\end{tabular}
}
\end{table}

\noindent\textbf{QA accuracy.} On ReviseQA, with a GPT-4o Interpreter producing every update from raw text, the two verifier arms differ by at most $2.26$ percentage points of accuracy@2 and $1.94$ of accuracy@7 across the five QA models, in both directions, and by at most $0.71$ percentage points in the per-model mean over $k\le7$. The end-to-end arm stays above RAG at every depth on all five models (\Cref{app:reviseqa-models}). On MemAB-FC, pooling the 6K and 32K streams, end-to-end single-hop and multi-hop accuracy differs from native accuracy by at most $1.5$ percentage points across Qwen2.5-7B, Qwen2.5-14B and Gemma-3-12B, and remains above RAG in every comparison.

\noindent\textbf{Faithfulness.} On ReviseQA the Interpreter emits a removal on $4{,}355$ of the $4{,}361$ edit steps that retire a premise inside the conclusion's maintained dependency set. On MemAB-FC, the dependency map the Interpreter builds from raw text matches the one built from the benchmark's structured updates at Jaccard $1.000$ at 6K and $0.989$ at 32K, with the Interpreter deciding which record each new fact supersedes (\Cref{app:capture,app:cr}).

\noindent\textbf{Prompt ablation.} \Cref{tab:cross-model} varies what the Interpreter is told: the operation names only, full definitions with examples, or the definitions plus the rendered engine state at checkpoints. For Claude Sonnet 4, definitions alone lift Phase~2 F1 from $0.66$ to $0.85$, so the taxonomy is learnable from general definitions, and the rendered state reaches $0.91$ overall and lifts the key shifts from $0.74$ to $0.83$. Interpreters differ widely under the same prompt, from $0.91$ to $0.59$ in the state-augmented condition. On Phase~3, evaluated under the Definitions condition only, Claude Sonnet 4 reaches $0.85$, and every Interpreter matches or exceeds its Phase~2 F1 under the same prompt (\Cref{tab:phase3-results}).

\subsection{Dependency closure}
\label{sec:closure}

\noindent\textbf{Closure versus one-step.} Given ground-truth dependency structure we compare the one-step and iterated queries of \Cref{def:affected}. Over the four retraction queries of the authored Phase~2 scenario, carrying $16$ dependency decisions, the one-step query returns the whole affected set for one of them, and $\mathit{Affected}^{*}$ for all four. Over $64$ counterfactual scenarios from $16$ RECON case files \citep{arya2026recon}, the one-step query is exact on $13$ and $\mathit{Affected}^{*}$ on all $64$, recovering $198$ dependent steps with no false positives (\Cref{app:recon}). This is \Cref{cor:dep-closure} in practice: given a faithful $\mathcal{D}_t$, the iterated query returns exactly the arguments that lose support.

\noindent\textbf{LLM reference.} When GPT-4o at temperature $0$ answers the same question, it is exact on all $64$ scenarios from the isolated chain of about $550$ tokens, and on $61$ from RECON's own oracle setting, the full $51$K-token skeleton, over-including one step in each of the other three. The iterated query computes the recorded dependency closure deterministically and without an LLM call.

\subsection{Efficiency}
\label{sec:efficiency}

\begin{figure}[t]
\centering
\includegraphics[width=\linewidth]{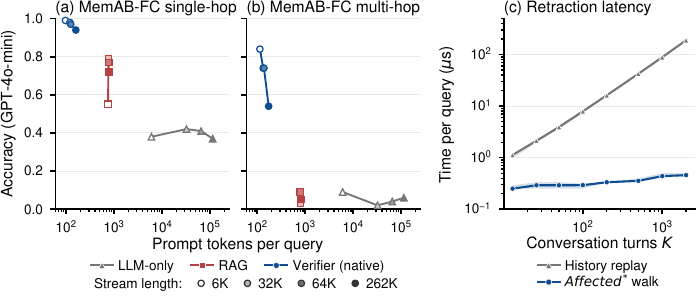}
\caption{\textbf{Query costs.} \textbf{(a, b)} MemAB-FC single-hop and multi-hop accuracy versus prompt tokens per query. \textbf{(c)} Median retraction latency versus history replay.}
\label{fig:cost}
\end{figure}

\noindent\textbf{Token cost.} On MemAB-FC, the LLM-only prompt grows with the stream, up to $114.5$K tokens. RAG retrieves the top-$50$ facts, $753$--$811$ tokens per query at every length. Verifier (native) sends only the active records linked to the question, $97$--$174$ tokens, $4.7$--$7.7\times$ fewer than RAG and $686\times$ fewer than LLM-only at 262K (\Cref{fig:cost}a,b). It is also the most accurate arm at every length for all five QA models. On BEAM, Verifier (end-to-end) reads $3.7$K tokens per question, about $3\%$ of GPT-4o's full-context input (\Cref{app:beam}). Verifier (native) takes the benchmark's structured updates and spends no LLM tokens at write time. Following the write and read accounting of \citet{pollertlam2026beyond}, Verifier (end-to-end) pays a one-time Interpreter ingestion cost of $96.6$K tokens at 6K and $494.7$K at 32K, which its per-query savings over LLM-only recover after $16.3$ and $15.3$ of the $100$ queries per stream (\Cref{app:e4}).

\noindent\textbf{Retraction latency.} With a reverse dependency index, a retraction query takes $0.25$--$0.46\,\mu$s over states of $13$--$2000$ turns, compared with $1.1$--$190\,\mu$s for history replay (\Cref{fig:cost}c). The gap widens with the conversation, from $4\times$ at $13$ turns to $415\times$ at $2000$. Latency scales with the size of the affected set rather than the history length, and \Cref{prop:complexity} bounds the worst case.

\section{Related work}
\label{sec:related}

\noindent\textbf{RAG and attribution.} RAG retrieves evidence for generation \citep{lewis2020rag}, and ALCE and AutoAIS score whether a generated answer is supported by the documents it cites \citep{gao2023alce,bohnet2023autoais}. Self-RAG adds retrieval and generation self-critique \citep{asai2024selfrag}, GraphRAG organises evidence into entity graphs \citep{edge2024graphrag}, and MTRAG evaluates retrieval and generation over multi-turn dialogue \citep{katsis2025mtrag}. These approaches anchor support in source documents or retrieved passages. The verifier checks interactional grounding instead: whether a continuation rests on commitments that still stand in the conversation.

\noindent\textbf{Conversational memory.} LLM agents keep memory in many forms \citep{zhang2024survey}. Mem0 consolidates conversational facts \citep{chhikara2025mem0}, A-MEM links and evolves memory notes \citep{xu2025amem}, Zep invalidates temporally conflicting facts \citep{rasmussen2025zep}, and D-SMART resolves conflicts in an evolving knowledge graph \citep{lei2025dsmart}. None of them specifies a mechanism for computing the transitive effects of a premise withdrawal on conclusions derived earlier. VISTA reconstructs a reasoning dependency tree from a conversation history for interactive analysis \citep{zhang2025vista}, without a runtime check of whether a continuation's support still stands. The verifier maintains dependencies and standing at every turn, answers a support query online, and returns exactly the recorded transitive dependents of a withdrawn premise (\Cref{app:comparison}).

\noindent\textbf{Formal foundations.} Truth-maintenance systems revise beliefs through recorded justifications \citep{doyle1979tms}, DEL models epistemic updates \citep{vanditmarsch2007del}, and argumentation frameworks supply conflict semantics \citep{dung1995acceptability}. Runtime verification checks an execution against a property as it runs \citep{leucker2009brief}, the lightweight counterpart of the model checking over multi-agent models \citep{huang2011bounded,lomuscio2017mcmas}. The verifier applies these foundations to one property, groundedness, over a state that an LLM Interpreter extracts from the conversation itself.

\section{Conclusion}
\label{sec:discussion}

We presented a runtime verifier that checks, at each turn, whether the next LLM output still rests on what the conversation has established. \Cref{lem:grounding} makes the check exact, and \Cref{lem:standing}, \Cref{prop:dep-sound} and \Cref{prop:complexity} establish conflict-free retraction and per-turn tractability. Given ground-truth dependencies, the closure query recovers every gold affected set in Phase~2 and RECON, including indirect dependents that one-step queries miss. The verifier leads RAG on ReviseQA and MemAB-FC across five QA models, and on BEAM where evidence was later revised. It reduces per-query token cost on MemAB-FC by \mbox{$4.7$--$7.7\times$} relative to RAG and by up to $686\times$ relative to LLM-only, and a retraction query takes under a microsecond over $13$--$2000$ turns.

% \noindent\textbf{Limitations.} The guarantees are structural: they concern standing and dependency, and leave the truth of individual claims to the Interpreter and the QA model, so extraction quality bounds what the structure can capture (\Cref{tab:cross-model}). The state also records every extracted statement as an assertion, without the speech act or provenance that would let a restatement link to its source instead of entering as a new premise.

\newpage
\bibliographystyle{iclr2027_conference}
\bibliography{references}

@inproceedings{laban2025lost,
  author    = {Philippe Laban and Hiroaki Hayashi and Yingbo Zhou and
               Jennifer Neville},
  title     = {{LLMs} Get Lost in Multi-Turn Conversation},
  booktitle = {International Conference on Learning Representations (ICLR)},
  year      = {2026},
}

@inproceedings{sirdeshmukh2025multichallenge,
  author    = {Ved Sirdeshmukh and Kaustubh Deshpande and Johannes Mols and
               Lifeng Jin and Ed-Yeremai Cardona and Dean Lee and Jeremy Kritz and
               Willow Primack and Summer Yue and Chen Xing},
  title     = {{MultiChallenge}: A Realistic Multi-Turn Conversation Evaluation
               Benchmark Challenging to Frontier {LLMs}},
  booktitle = {Findings of the Association for Computational Linguistics: ACL 2025},
  pages     = {18632--18702},
  year      = {2025},
  publisher = {Association for Computational Linguistics},
  doi       = {10.18653/v1/2025.findings-acl.958},
  url       = {https://aclanthology.org/2025.findings-acl.958/},
}

@inproceedings{wilie2024belief,
  author    = {Bryan Wilie and Samuel Cahyawijaya and Etsuko Ishii and
               Junxian He and Pascale Fung},
  title     = {Belief Revision: The Adaptability of Large Language Models Reasoning},
  booktitle = {Proceedings of the 2024 Conference on Empirical Methods in
               Natural Language Processing},
  pages     = {10480--10496},
  year      = {2024},
  publisher = {Association for Computational Linguistics},
  doi       = {10.18653/v1/2024.emnlp-main.586},
  url       = {https://aclanthology.org/2024.emnlp-main.586/},
}

@article{wang2025agentarmor,
  title     = {{AgentArmor}: Enforcing Program Analysis on Agent Runtime Trace to Defend Against Prompt Injection},
  author    = {Wang, Peiran and Liu, Yang and Lu, Yunfei and Cai, Yifeng and Chen, Hongbo and Yang, Qingyou and Zhang, Jie and Hong, Jue and Wu, Ye},
  journal   = {arXiv preprint arXiv:2508.01249},
  year      = {2025}
}

@incollection{baltag2008qualitative,
  author    = {Alexandru Baltag and Sonja Smets},
  title     = {A Qualitative Theory of Dynamic Interactive Belief Revision},
  booktitle = {Logic and the Foundations of Game and Decision Theory
               ({LOFT}~7)},
  editor    = {Giacomo Bonanno and Wiebe van der Hoek and
               Michael Wooldridge},
  series    = {Texts in Logic and Games},
  volume    = {3},
  pages     = {11--58},
  publisher = {Amsterdam University Press},
  year      = {2008},
}

@inproceedings{baltag1998logic,
  author    = {Alexandru Baltag and Lawrence S. Moss and
               Slawomir Solecki},
  title     = {The Logic of Public Announcements, Common Knowledge,
               and Private Suspicions},
  booktitle = {Proceedings of the 7th Conference on Theoretical Aspects
               of Rationality and Knowledge ({TARK})},
  pages     = {43--56},
  publisher = {Morgan Kaufmann},
  year      = {1998},
}

@book{fagin1995reasoning,
  author    = {Ronald Fagin and Joseph Y. Halpern and Yoram Moses and
               Moshe Y. Vardi},
  title     = {Reasoning about Knowledge},
  publisher = {MIT Press},
  address   = {Cambridge, MA},
  year      = {1995},
}

@article{fagin1988awareness,
  author    = {Ronald Fagin and Joseph Y. Halpern},
  title     = {Belief, Awareness, and Limited Reasoning},
  journal   = {Artificial Intelligence},
  volume    = {34},
  number    = {1},
  pages     = {39--76},
  year      = {1987},
}

@book{vanditmarsch2007del,
  author    = {Hans van Ditmarsch and Wiebe van der Hoek and
               Barteld Kooi},
  title     = {Dynamic Epistemic Logic},
  series    = {Synthese Library},
  volume    = {337},
  publisher = {Springer},
  address   = {Dordrecht},
  year      = {2007},
}

@article{velazquez2013abductive,
  author    = {Fernando R. Vel{\'a}zquez-Quesada and
               Fernando Soler-Toscano and
               {\'A}ngel Nepomuceno-Fern{\'a}ndez},
  title     = {An Epistemic and Dynamic Approach to Abductive Reasoning:
               Abductive Problem and Abductive Solution},
  journal   = {Journal of Applied Logic},
  volume    = {11},
  number    = {4},
  pages     = {505--522},
  year      = {2013},
}

@inproceedings{aucher2013complexity,
  author    = {Guillaume Aucher and Fran\c{c}ois Schwarzentruber},
  title     = {On the Complexity of Dynamic Epistemic Logic},
  booktitle = {Proceedings of the 14th Conference on Theoretical Aspects
               of Rationality and Knowledge ({TARK})},
  year      = {2013},
}

@incollection{dunne2009complexity,
  author    = {Paul E. Dunne and Michael Wooldridge},
  title     = {Complexity of Abstract Argumentation},
  booktitle = {Argumentation in Artificial Intelligence},
  editor    = {Guillermo R. Simari and Iyad Rahwan},
  publisher = {Springer},
  pages     = {85--104},
  year      = {2009},
}

@article{dunne2007computational,
  author    = {Paul E. Dunne},
  title     = {Computational Properties of Argument Systems Satisfying
               Graph-theoretic Constraints},
  journal   = {Artificial Intelligence},
  volume    = {171},
  number    = {10--15},
  pages     = {701--729},
  year      = {2007},
}

@article{dung1995acceptability,
  author    = {Phan Minh Dung},
  title     = {On the Acceptability of Arguments and its Fundamental Role
               in Nonmonotonic Reasoning, Logic Programming and
               $n$-Person Games},
  journal   = {Artificial Intelligence},
  volume    = {77},
  number    = {2},
  pages     = {321--357},
  year      = {1995},
}

@inproceedings{doutre2014del,
  author    = {Sylvie Doutre and Andreas Herzig and
               Laurent Perrussel},
  title     = {A Dynamic Logic Framework for Abstract Argumentation},
  booktitle = {Proceedings of the Fourteenth International Conference on
               Principles of Knowledge Representation and Reasoning
               ({KR})},
  publisher = {AAAI Press},
  pages     = {62--71},
  year      = {2014},
}

@book{walton2008argumentation,
  author    = {Douglas Walton and Chris Reed and Fabrizio Macagno},
  title     = {Argumentation Schemes},
  publisher = {Cambridge University Press},
  year      = {2008},
}

@techreport{kunz1970ibis,
  author      = {Werner Kunz and Horst W. J. Rittel},
  title       = {Issues as Elements of Information Systems},
  institution = {Institute of Urban and Regional Development,
                 University of California, Berkeley},
  number      = {131},
  year        = {1970},
}

@article{lomuscio2017mcmas,
  author    = {Alessio Lomuscio and Hongyang Qu and Franco Raimondi},
  title     = {{MCMAS}: An Open-Source Model Checker for the Verification
               of Multi-Agent Systems},
  journal   = {International Journal on Software Tools for Technology
               Transfer},
  volume    = {19},
  number    = {1},
  pages     = {9--30},
  year      = {2017},
  note      = {Earlier tool papers at TACAS 2006 and CAV 2009; often
               cited with earlier dates},
}

@inproceedings{huang2011bounded,
  author    = {Xiaowei Huang and Cheng Luo and Ron van der Meyden},
  title     = {Improved Bounded Model Checking for a Fair Branching-Time
               Temporal Epistemic Logic},
  booktitle = {Model Checking and Artificial Intelligence (MoChArt 2010)},
  series    = {Lecture Notes in Computer Science},
  volume    = {6572},
  pages     = {95--111},
  publisher = {Springer},
  year      = {2011},
  note      = {Feasibility of bounded epistemic model checking},
}

@inproceedings{park2023generative,
  author    = {Joon Sung Park and Joseph C. O'Brien and
               Carrie J. Cai and Meredith Ringel Morris and
               Percy Liang and Michael S. Bernstein},
  title     = {Generative Agents: Interactive Simulacra of Human
               Behavior},
  booktitle = {Proceedings of the 36th Annual {ACM} Symposium on
               User Interface Software and Technology ({UIST}~'23)},
  pages     = {1--22},
  publisher = {ACM},
  year      = {2023},
}

@article{zhang2024survey,
  author    = {Zeyu Zhang and Xiaohe Bo and Chen Ma and Rui Li and
               Xu Chen and Quanyu Dai and Jieming Zhu and
               Zhenhua Dong and Ji-Rong Wen},
  title     = {A Survey on the Memory Mechanism of Large Language
               Model Based Agents},
  journal   = {arXiv preprint arXiv:2404.13501},
  year      = {2024},
}

@article{lei2025dsmart,
  author    = {Xiang Lei and Qin Li and Min Zhang and Min Zhang},
  title     = {{D-SMART}: Enhancing {LLM} Dialogue Consistency via
               Dynamic Structured Memory and Reasoning Tree},
  journal   = {arXiv preprint arXiv:2510.13363},
  year      = {2025},
}

@article{rasmussen2025zep,
  author    = {Preston Rasmussen and Pavlo Paliychuk and
               Travis Beauvais and Jack Ryan and Daniel Chalef},
  title     = {Zep: A Temporal Knowledge Graph Architecture for
               Agent Memory},
  journal   = {arXiv preprint arXiv:2501.13956},
  year      = {2025},
}

@inproceedings{obiso2025friction,
  author    = {Timothy Obiso and Kenneth Lai and Abhijnan Nath and
               Nikhil Krishnaswamy and James Pustejovsky},
  title     = {Dynamic Epistemic Friction in Dialogue},
  booktitle = {Proceedings of the 29th Conference on Computational
               Natural Language Learning ({CoNLL})},
  pages     = {323--333},
  publisher = {Association for Computational Linguistics},
  address   = {Vienna, Austria},
  year      = {2025},
}

@inproceedings{helwe2025reviseqa,
  title={{ReviseQA}: A Benchmark for Belief Revision in Multi-Turn Logical Reasoning},
  author={Helwe, Chadi and AlRashed, Sultan and Orabona, Francesco},
  booktitle={ICML 2025 Workshop on Assessing World Models},
  year={2025}
}

@inproceedings{asai2024selfrag,
  title     = {{Self-RAG}: Learning to Retrieve, Generate, and Critique through Self-Reflection},
  author    = {Asai, Akari and Wu, Zeqiu and Wang, Yizhong and Sil, Avirup and Hajishirzi, Hannaneh},
  booktitle = {International Conference on Learning Representations ({ICLR})},
  year      = {2024},
  url       = {https://openreview.net/forum?id=hSyW5go0v8},
}

@article{edge2024graphrag,
  title     = {From Local to Global: A {Graph RAG} Approach to Query-Focused Summarization},
  author    = {Edge, Darren and Trinh, Ha and Cheng, Newman and Bradley, Joshua and Chao, Alex and Mody, Apurva and Truitt, Steven and Metropolitansky, Dasha and Ness, Robert Osazuwa and Larson, Jonathan},
  journal   = {arXiv preprint arXiv:2404.16130},
  year      = {2024},
  url       = {https://arxiv.org/abs/2404.16130},
}

@inproceedings{xu2025amem,
  title     = {{A-Mem}: Agentic Memory for {LLM} Agents},
  author    = {Xu, Wujiang and Liang, Zujie and Mei, Kai and Gao, Hang and Tan, Juntao and Zhang, Yongfeng},
  booktitle = {Advances in Neural Information Processing Systems},
  volume    = {38},
  year      = {2025},
  doi       = {10.52202/085713-0593},
  url       = {https://proceedings.neurips.cc/paper_files/paper/2025/hash/19909c36f51abc4856b4560aff3d36d6-Abstract-Conference.html},
}

@inproceedings{gao2023alce,
  title={Enabling Large Language Models to Generate Text with Citations},
  author={Gao, Tianyu and Yen, Howard and Yu, Jiatong and Chen, Danqi},
  booktitle={Proceedings of the 2023 Conference on Empirical Methods in Natural Language Processing (EMNLP)},
  pages={6465--6488},
  year={2023}
}

@article{bohnet2023autoais,
  title={Attributed Question Answering: Evaluation and Modeling for Attributed Large Language Models},
  author={Bohnet, Bernd and Tran, Vinh Q. and Verga, Pat and Aharoni, Roee and Andor, Daniel and Soares, Livio Baldini and Ciaramita, Massimiliano and Eisenstein, Jacob and Ganchev, Kuzman and Herzig, Jonathan and Hui, Kai and others},
  journal={arXiv preprint arXiv:2212.08037},
  year={2022}
}

@inproceedings{shaikh2025rifts,
  title={Navigating Rifts in Human-{LLM} Grounding: Study and Benchmark},
  author={Shaikh, Omar and Mozannar, Hussein and Bansal, Gagan and Fourney, Adam and Horvitz, Eric},
  booktitle={Proceedings of the 63rd Annual Meeting of the Association for Computational Linguistics (Volume 1: Long Papers)},
  pages={20832--20847},
  year={2025}
}

@article{katsis2025mtrag,
  title={{mtRAG}: A Multi-Turn Conversational Benchmark for Evaluating Retrieval-Augmented Generation Systems},
  author={Katsis, Yannis and Rosenthal, Sara and Fadnis, Kshitij and Gunasekara, Chulaka and Lee, Young-Suk and Popa, Lucian and Shah, Vraj and Zhu, Huaiyu and Contractor, Danish and Danilevsky, Marina},
  journal={Transactions of the Association for Computational Linguistics},
  volume={13},
  pages={784--808},
  year={2025}
}

@article{patlan2025fakememories,
  title={Real {AI} Agents with Fake Memories: Fatal Context Manipulation Attacks on {Web3} Agents},
  author={Patlan, Atharv Singh and Sheng, Peiyao and Hebbar, S. Ashwin and Mittal, Prateek and Viswanath, Pramod},
  journal={arXiv preprint arXiv:2503.16248},
  year={2025}
}

@inproceedings{dong2025memoryinjection,
  title={Memory Injection Attacks on {LLM} Agents via Query-Only Interaction},
  author={Dong, Shen and Xu, Shaochen and He, Pengfei and Li, Yige and Tang, Jiliang and Liu, Tianming and Liu, Hui and Xiang, Zhen},
  booktitle={Advances in Neural Information Processing Systems ({NeurIPS})},
  year={2025},
  note={arXiv:2503.03704}
}

@inproceedings{zhang2025vista,
  title     = {Beyond the Black Box: Demystifying Multi-Turn {LLM} Reasoning with {VISTA}},
  author    = {Zhang, Yiran and Lin, Mingyang and Dras, Mark and Naseem, Usman},
  booktitle = {Proceedings of the {AAAI} Conference on Artificial Intelligence},
  volume    = {40},
  number    = {48},
  pages     = {41745--41747},
  year      = {2026},
  note      = {Demonstration Track},
}

@inproceedings{chhikara2025mem0,
  title     = {{Mem0}: Building Production-Ready {AI} Agents with Scalable Long-Term Memory},
  author    = {Chhikara, Prateek and Khant, Dev and Aryan, Saket and Singh, Taranjeet and Yadav, Deshraj},
  booktitle = {Proceedings of the 28th European Conference on Artificial Intelligence ({ECAI})},
  pages     = {2993--3000},
  year      = {2025},
  doi       = {10.3233/FAIA251160}
}

@inproceedings{hu2025memoryagentbench,
  title={Evaluating Memory in {LLM} Agents via Incremental Multi-Turn Interactions},
  author={Hu, Yuanzhe and Wang, Yu and McAuley, Julian},
  booktitle={International Conference on Learning Representations ({ICLR})},
  year={2026}
}

@inproceedings{bergstrasser2026succinct,
  title={Transformers are Inherently Succinct},
  author={Bergstr{\"a}{\ss}er, Pascal and Cotterell, Ryan and Lin, Anthony W.},
  booktitle={International Conference on Learning Representations ({ICLR})},
  year={2026}
}

@inproceedings{huang2017safety,
  title={Safety Verification of Deep Neural Networks},
  author={Huang, Xiaowei and Kwiatkowska, Marta and Wang, Sen and Wu, Min},
  booktitle={International Conference on Computer Aided Verification ({CAV})},
  pages={3--29},
  year={2017}
}

@article{huang2020survey,
  title={A Survey of Safety and Trustworthiness of Deep Neural Networks: Verification, Testing, Adversarial Attack and Defence, and Interpretability},
  author={Huang, Xiaowei and Kroening, Daniel and Ruan, Wenjie and Sharp, James and Sun, Youcheng and Thamo, Emese and Wu, Min and Yi, Xinping},
  journal={Computer Science Review},
  volume={37},
  pages={100270},
  year={2020}
}

@article{arya2026recon,
  title={{RECON}: Benchmarking Agent Memory for Compositional Reasoning over Long Contexts},
  author={Arya, Mihir Shriniwas},
  journal={arXiv preprint arXiv:2607.16716},
  year={2026}
}

@article{pollertlam2026beyond,
  title={Beyond the context window: A cost-performance analysis of fact-based memory vs. long-context LLMs for persistent agents},
  author={Pollertlam, Natchanon and Kornsuwannawit, Witchayut},
  journal={arXiv preprint arXiv:2603.04814},
  year={2026}
}

@article{doyle1979tms,
  author  = {Doyle, Jon},
  title   = {A Truth Maintenance System},
  journal = {Artificial Intelligence},
  volume  = {12},
  number  = {3},
  pages   = {231--272},
  year    = {1979}
}

@article{leucker2009brief,
  author  = {Leucker, Martin and Schallhart, Christian},
  title   = {A Brief Account of Runtime Verification},
  journal = {Journal of Logic and Algebraic Programming},
  volume  = {78},
  number  = {5},
  pages   = {293--303},
  year    = {2009}
}

@inproceedings{tavakoli2026beam,
  title={Beyond a Million Tokens: Benchmarking and Enhancing Long-Term Memory in {LLMs}},
  author={Tavakoli, Mohammad and Salemi, Alireza and Ye, Carrie and Abdalla, Mohamed and Zamani, Hamed and Mitchell, J Ross},
  booktitle={International Conference on Learning Representations (ICLR)},
  year={2026}
}

@inproceedings{hu2026lying,
  title={Lying with truths: Open-channel multi-agent collusion for belief manipulation via generative montage},
  author={Hu, Jinwei and Huang, Xinmiao and Sun, Youcheng and Dong, Yi and Huang, Xiaowei},
  booktitle={Proceedings of the 64th Annual Meeting of the Association for Computational Linguistics (Volume 1: Long Papers)},
  pages={5979--5996},
  year={2026}
}

@inproceedings{lewis2020rag,
  title={Retrieval-Augmented Generation for Knowledge-Intensive {NLP} Tasks},
  author={Lewis, Patrick and Perez, Ethan and Piktus, Aleksandra and Petroni, Fabio and Karpukhin, Vladimir and Goyal, Naman and K{\"u}ttler, Heinrich and Lewis, Mike and Yih, Wen-tau and Rockt{\"a}schel, Tim and Riedel, Sebastian and Kiela, Douwe},
  booktitle={Advances in Neural Information Processing Systems},
  volume={33},
  year={2020}
}

@article{yang2024qwen25,
  title   = {{Qwen2.5} Technical Report},
  author  = {Yang, An and Yang, Baosong and Zhang, Beichen and Hui, Binyuan and Zheng, Bo and Yu, Bowen and Li, Chengyuan and Liu, Dayiheng and Huang, Fei and Wei, Haoran and others},
  journal = {arXiv preprint arXiv:2412.15115},
  year    = {2024},
}

@article{gemmateam2025gemma3,
  title   = {{Gemma} 3 Technical Report},
  author  = {{Gemma Team} and Kamath, Aishwarya and Ferret, Johan and Pathak, Shreya and Vieillard, Nino and Merhej, Ramona and Perrin, Sarah and Matejovicova, Tatiana and Ram{\'e}, Alexandre and Rivi{\`e}re, Morgane and others},
  journal = {arXiv preprint arXiv:2503.19786},
  year    = {2025},
}

@article{openai2024gpt4o,
  title   = {{GPT-4o} System Card},
  author  = {Hurst, Aaron and Lerer, Adam and Goucher, Adam P and Perelman, Adam and Ramesh, Aditya and Clark, Aidan and Ostrow, AJ and Welihinda, Akila and Hayes, Alan and Radford, Alec and others},
  journal = {arXiv preprint arXiv:2410.21276},
  year    = {2024},
}

@misc{openai2024gpt4omini,
  author       = {{OpenAI}},
  title        = {{GPT-4o} mini: Advancing Cost-Efficient Intelligence},
  howpublished = {OpenAI Blog,
                  \url{https://openai.com/index/gpt-4o-mini-advancing-cost-efficient-intelligence/}},
  year         = {2024},
}

@misc{anthropic2025claude4,
  author       = {{Anthropic}},
  title        = {System Card: {Claude Opus 4} \& {Claude Sonnet 4}},
  howpublished = {\url{https://www.anthropic.com/claude-4-system-card}},
  year         = {2025},
}

@article{grattafiori2024llama3,
  title   = {The {Llama} 3 Herd of Models},
  author  = {Grattafiori, Aaron and Dubey, Abhimanyu and Jauhri, Abhinav and Pandey, Abhinav and Kadian, Abhishek and others},
  journal = {arXiv preprint arXiv:2407.21783},
  year    = {2024},
}

@inproceedings{zheng2023judging,
  title     = {Judging {LLM}-as-a-Judge with {MT-Bench} and {Chatbot Arena}},
  author    = {Zheng, Lianmin and Chiang, Wei-Lin and Sheng, Ying and Zhuang, Siyuan and Wu, Zhanghao and Zhuang, Yonghao and Lin, Zi and Li, Zhuohan and Li, Dacheng and Xing, Eric P. and Zhang, Hao and Gonzalez, Joseph E. and Stoica, Ion},
  booktitle = {Advances in Neural Information Processing Systems, Datasets and Benchmarks Track},
  volume    = {36},
  year      = {2023},
}

@article{mcnemar1947note,
  title   = {Note on the Sampling Error of the Difference between Correlated Proportions or Percentages},
  author  = {McNemar, Quinn},
  journal = {Psychometrika},
  volume  = {12},
  number  = {2},
  pages   = {153--157},
  year    = {1947},
}

@article{wilcoxon1945individual,
  title   = {Individual Comparisons by Ranking Methods},
  author  = {Wilcoxon, Frank},
  journal = {Biometrics Bulletin},
  volume  = {1},
  number  = {6},
  pages   = {80--83},
  year    = {1945},
}

\clearpage
\appendix
% \Cref prints "Appendix" for sections after \appendix
\crefalias{section}{appendix}
\crefalias{subsection}{subappendix}
\crefalias{subsubsection}{subsubappendix}
% Appendix float placement; \topfraction and \textfraction keep the values set by the style file
\renewcommand{\bottomfraction}{0.8}\renewcommand{\floatpagefraction}{0.8}
\setcounter{topnumber}{4}\setcounter{bottomnumber}{3}\setcounter{totalnumber}{6}

\section{Formal Framework and Proofs}
\label{app:proofs}

\FloatBarrier
\subsection{Formal Definitions Deferred from Section~\ref{sec:framework}}

\begin{definition}[Epistemic Plausibility Model]
\label{def:epm}
An \emph{epistemic plausibility model} is a tuple $\Model = (W, \{\leq_i\}_{i \in \mathrm{Ags}}, V)$ where:
\begin{itemize}[leftmargin=*,itemsep=0pt]
    \item $W$ is a non-empty set of possible worlds;
    \item each ${\leq_i} \subseteq W \times W$ is a reflexive, transitive, \emph{locally connected} preorder: for all $w, w'$ in the same indistinguishability cell, either $w \leq_i w'$ or $w' \leq_i w$;
    \item $V : \mathrm{Prop} \to \mathcal{P}(W)$ is a valuation.
\end{itemize}
\end{definition}

\noindent\textbf{Convention.} We follow the standard Baltag--Smets reading: $w \leq_i w'$ means ``$w$ is at least as plausible as $w'$ for agent $i$,'' so $\leq_i$-minimal worlds are most plausible. The indistinguishability relation is $\sim_i \;:=\; \leq_i \cup \geq_i$. Different agents may have different information cells from the same $w$. Agent $i$ \emph{knows} $\varphi$ at $w$ iff $\varphi$ holds at every $w'$ with $w \sim_i w'$ \citep{fagin1995reasoning}. Agent $i$ \emph{believes} $\varphi$ at $w$ iff $\varphi$ holds at every $\leq_i$-minimal world in $\{w' : w \sim_i w'\}$.

\noindent\textbf{Updates.} The three dynamic operators in full: \emph{Public announcement} (hard) $\announce{!\psi}$ eliminates every $w \notin \llbracket \psi \rrbracket$. \emph{Lexicographic upgrade} (radical, soft) $\announce{\Uparrow \psi}$ makes every $\psi$-world strictly more plausible than every $\lnot\psi$-world while preserving the within-side order. No worlds are eliminated. \emph{Conservative upgrade} (minimal, soft) $\announce{\uparrow \psi}$ promotes the most plausible $\psi$-worlds to the overall minimum, preserving all other comparisons. Event models \citep{baltag1998logic} generalise these when different agents perceive the same event differently.

\begin{definition}[Abductive Problem and Solution]
\label{def:abduction}
Given a model $\Model$, actual world $w$, agent $i$, and observation $\chi$ with $\Model, w \models \chi$ but $\Model, w \not\models \B{i}\chi$ (the observation is true but not previously believed, \emph{surprising}), an \emph{abductive problem} is the triple $(\Model, i, \chi)$. A formula $\gamma$ is an \emph{abductive solution} when:
\begin{enumerate}[leftmargin=*,itemsep=0pt,label=(\arabic*)]
    \item \emph{Consistency}: $\Model \not\models \B{i} \lnot \gamma$;
    \item \emph{Explanatory}: $\Model\announce{\Uparrow \gamma}, w \models \B{i}\chi$;
    \item \emph{Non-triviality}: $\gamma \neq \chi$, $\Model \not\models \B{i}\gamma$, and $\gamma \neq \top$.
\end{enumerate}
The solution is integrated via $\announce{\Uparrow \gamma}$, so $\gamma$ enters as belief, not knowledge.
\end{definition}

\begin{definition}[Awareness Structure]
\label{def:awareness}
An \emph{awareness structure} \citep{fagin1988awareness} extends an epistemic plausibility model with a function $\mathcal{A}_i : W \to \mathcal{P}(\text{Form})$ for each agent $i$, where $\mathcal{A}_i(w)$ is the set of formulas $i$ is aware of at $w$. The awareness modality satisfies $\Model, w \models \A{i}\varphi \iff \varphi \in \mathcal{A}_i(w)$. Explicit knowledge is $X_i\varphi \iff \K{i}\varphi \land \A{i}\varphi$. Awareness expansion adds formulas to $\mathcal{A}_i$ and refines $W$ accordingly: each world splits into copies that differ on the new proposition.
\end{definition}

\paragraph{Recorded support and traversal.} Each member of $\mathit{Dep}_t(\alpha)$ is a supporting record identified by its ID. Different IDs remain distinct even when their claims coincide. The associated ID is fixed before traversal, and a candidate with no associated record is ungrounded. The formal check assumes that this association identifies the candidate's claim. The walk follows the recorded IDs and checks their current labels, including those of weakened or abandoned hypotheses. A visited set prevents repeated expansion. Records with no dependencies end a branch. Retraction queries traverse the same relation in reverse and return reports without changing standing. In the query bounds, $|\mathit{Dep}_t|:=\sum_{\alpha\in\mathit{Args}_t}|\mathit{Dep}_t(\alpha)|$.

\FloatBarrier
\subsection{Proofs}
%==================================================================

\begin{proof}[Proof of \Cref{lem:grounding} (Grounding correctness)]
Fix the candidate's associated record $\alpha_c$ and the finite map $\mathcal{D}_t$. The walk begins at $\alpha_c$ and expands exactly the recorded dependency links, so every visited argument lies in $\mathit{Dep}^{*}_t(\alpha_c)$. Conversely, if the walk does not fail, induction on dependency-path length shows that it visits every member of this closure. A visited set ensures termination even when dependency links form cycles. Every record is checked against $S_t$ on its first visit. If all closure members lie in $S_t$, no check fails and the walk returns the complete closure. If some member lies outside $S_t$, the walk either fails earlier or reaches such a member and rejects. Therefore acceptance is equivalent to $\mathit{Dep}^{*}_t(\alpha_c)\subseteq S_t$, and the returned set on acceptance is exactly the closure.
\end{proof}

\begin{proof}[Proof of \Cref{lem:standing} (Standing invariant)]
By induction on turns. Initially $\mathit{Att}_0 = \emptyset$. Inspecting \Cref{alg:apply}: \textsc{Observe}, \textsc{Hypothesize}, \textsc{Support}, \textsc{Expand-Awareness} and \textsc{Question} add no attack edge and label any new argument $\mathsf{active}$; \textsc{Undermine}, whose precondition requires an active or weakened target, adds $(\alpha_e, \alpha_\gamma)$ and leaves $\mathit{st}(\alpha_\gamma) = \mathsf{weakened}$; \textsc{Revise} adds attacks on $\alpha_\gamma$ and sets $\mathit{st}(\alpha_\gamma) = \mathsf{abandoned}$; \textsc{Resolve} sets $\mathit{st}(\alpha_\gamma) = \mathsf{resolved}$ only when $\alpha_\gamma \in S_t$ and no attack on it is recorded (\Cref{tab:ops-full}); and no operation moves a weakened or abandoned argument back into $S_t$. Hence after every turn each attack targets an argument outside $S_t$, so no two arguments of $S_t$ attack each other.
\end{proof}

\begin{proof}[Proof of \Cref{prop:dep-sound} (Conflict-free selective retraction)]
Let $\mathcal{D}_t$ be a dependency map satisfying \Cref{lem:standing}, with the withdrawn premise record $\beta \notin S_t$. Define $R := \{\beta\} \cup \mathit{Affected}(\beta)$, $S' := S_t \setminus R$, and $\mathit{AF}' := \mathit{AF}_t \!\restriction\! R$. Its argument set is $\mathit{Args}_t \setminus R$ and its attack relation is $\mathit{Att}_t \cap (\mathit{Args}_t \setminus R)^2$.

\emph{Conflict-freeness and preservation}: $S_t$ is conflict-free by \Cref{lem:standing}, and $S' \subseteq S_t$. Restriction cannot create new attacks, so $S'$ is conflict-free in $\mathit{AF}'$. Since $\beta \notin S_t$, an argument $\alpha \in S_t$ is removed precisely when $\beta \in \mathit{Dep}_t(\alpha)$. Thus the one-step removal preserves exactly those members of $S_t$ without a direct recorded dependency on $\beta$.

\emph{Admissibility and preferred extensions}: no argument of $S_t$ is the target of any attack in $\mathit{AF}_t$ by \Cref{lem:standing}. Hence no argument of $S'$ is attacked in $\mathit{AF}'$, and $S'$ is admissible: it defends itself vacuously. Every admissible set is contained in some preferred extension \citep{dung1995acceptability}. Such an extension contains every preserved member of $S_t$ and excludes $R$, whose records are absent from the restricted framework. The engine maintains standing directly, so no extension is computed at runtime.
\end{proof}

\begin{proof}[Proof of \Cref{cor:dep-closure} (Exact iterated retraction)]
Use the dependency-edge direction $\alpha \to \delta$ iff $\delta \in \mathit{Dep}_t(\alpha)$. Here a dependency path is a finite sequence of recorded links, possibly revisiting vertices. We show by induction that $X_k$ consists exactly of arguments having a path of length between $1$ and $k+1$ to $\beta$. For $k=0$, this is the definition of $\mathit{Affected}(\beta)$. At the next round, a newly included argument has a dependency in $X_k$, extending such a path by one edge. Conversely, any path of length at most $k+2$ either has length one or begins with an edge to an argument in $X_k$, so its starting argument lies in $X_{k+1}$. Taking the union proves that $\mathit{Affected}^{*}(\beta)$ is exactly the set with a nonempty recorded path to $\beta$. The increasing sequence is bounded by finite $\mathit{Args}_t$ and stabilises within $|\mathit{Args}_t|$ rounds. This argument also covers dependency cycles.

Let $R^{*} := \{\beta\} \cup \mathit{Affected}^{*}(\beta)$. For each $\alpha \neq \beta$, membership in $R^{*}$ is equivalent to having such a path, so removing $R^{*}$ removes exactly the withdrawn record and its recorded dependents and preserves every other argument. Under the hypotheses of \Cref{prop:dep-sound}, $S_t$ is conflict-free by \Cref{lem:standing}, so its subset $S_t \setminus R^{*}$ is conflict-free in $\mathit{AF}_t \!\restriction\! R^{*}$, since restricting the attack relation introduces no conflict. The exact-reachability statement holds for any map.
\end{proof}

\begin{proof}[Proof of \Cref{prop:complexity} (Tractability under structural restrictions)]
\emph{Hard announcements}: Each public announcement $\announce{!\psi}$ eliminates worlds where $\lnot\psi$ holds, requiring one pass through $W$: $\bigO(|W|)$. With $k$ binary propositions, $|W| \leq 2^k$. Over $T$ turns: $\bigO(T \cdot 2^k)$.

\emph{Soft upgrades}: A lexicographic upgrade $\announce{\Uparrow\gamma}$ (definition in the \emph{Updates} paragraph above) can be realised by partitioning $W$ on $\llbracket\gamma\rrbracket$ and stably sorting on $\leq_i$, in $\bigO(|W| \log |W|)$ per upgrade. With at most $h$ hypotheses over $T$ turns this bound gives $\bigO(T \cdot h \cdot |W| \log |W|) = \bigO(T \cdot h \cdot k \cdot 2^k)$. Awareness expansion adds atomic propositions to $\mathit{Prop}$, potentially doubling $|W|$ each time, so $e$ expansions replace $k$ by $k+e$ in this bound.

\emph{Standing and queries}: Standing is maintained by the operations (\Cref{lem:standing}) at $\bigO(1)$ per label update, so no extension is recomputed at runtime. A visited-set traversal for $\mathsf{Verify}$ checks each reachable argument and its recorded dependencies at most once. The $\mathit{Affected}^{*}(\beta)$ traversal does the same over the reverse index, visiting $\{\beta\} \cup \mathit{Affected}^{*}(\beta)$ and the reverse edges leaving those records. Each traversal therefore takes $\bigO(|\mathit{Args}| + |\mathit{Dep}|)$ time in the worst case. The reverse traversal takes $\bigO(1 + |\mathit{Affected}^{*}(\beta)| + m_\beta)$, where $m_\beta$ counts the reverse edges inspected. Each visited identifier is output once. For reference, when the attack graph $\mathit{Att}$ is acyclic, the grounded, complete, preferred, and stable semantics all yield the same unique extension \citep{dung1995acceptability}, computable in $\bigO(|\mathit{Args}| + |\mathit{Att}|)$ by topological processing \citep{dunne2007computational}. Without the reverse index, computing $\mathit{Affected}(\beta)$ by scanning the dependency records takes $\bigO(|\mathit{Args}| + |\mathit{Dep}|)$ time.

\emph{In our scenarios}, $k$ is 3 (muddy children), 15 (Phase~2 debugging: 7 observable propositions, $\mathit{mis\_monitor}$, $h_1$--$h_4$ and three causal atoms), and 11 (Phase~3 deliberation: 5 positions, 4 assumptions and 2 libraries). In all three, $|\mathit{Args}| \leq 25$.

\emph{General worst cases.} Without these restrictions, DEL model-checking with event models is \PSPACE-complete \citep{aucher2013complexity}, skeptical preferred acceptance in Dung-style AFs is $\PiP{2}$-complete \citep{dunne2009complexity}, and preferred-extension enumeration has no known polynomial-delay algorithm. The polynomial bounds above hold under (A1)--(A3) of \Cref{prop:complexity}, which exclude these cases.
\end{proof}

\FloatBarrier
\subsection{Per-Operation Updates and Per-Turn Composition: \texttt{Apply}}
\label{app:algorithm}

The Interpreter can extract multiple operation instances from one utterance. Each instance is applied to the four components of the dependency map (\Cref{def:eam}) using \Cref{alg:apply}: the epistemic plausibility model $\Model$, the argumentation framework $\mathit{AF}$, the commitment record $\mathit{Cm}$, and the dependency map $\mathit{Dep}$. For $m$ instances from an utterance by speaker $i$, the per-turn update is their sequential composition:
\begin{equation*}
\mathcal{D}_t^{(0)} = \mathcal{D}_t, \qquad
\mathcal{D}_t^{(j)} = \mathsf{Apply}(op_j,\mathit{args}_j,\mathcal{D}_t^{(j-1)},i) \quad (1 \leq j \leq m), \qquad
\mathcal{D}_{t+1} = \mathcal{D}_t^{(m)}.
\end{equation*}
The Interpreter supplies the identifiers of the records each operation acts on, with unused identifiers for new observations and hypotheses. Instances are applied in the order of \Cref{tab:ops-full}, preserving extraction order within each operation type. The guarantees of \Cref{lem:standing} hold for instances that meet the preconditions of \Cref{tab:ops-full}. In \Cref{alg:apply}, $\iota$ is the supplied identifier for a record-creation operation, and $\mathsf{recordArg}(\iota,\ldots)$ registers the supplied content at that ID. The notation $\alpha_\gamma$ denotes an existing record selected by ID with claim $\gamma$, and $\alpha_e$ denotes the record supplying evidence $e$. The algorithm presents the operation effects, and \Cref{tab:ops-full} states their preconditions.

\begin{algorithm}[!htbp]
\caption{$\mathsf{Apply}(op,\, \mathit{args},\, \mathcal{D}_t,\, i) \to \mathcal{D}_{t+1}$. Per-operation update rule: one extracted operation instance from speaker~$i$ updates the four components ($\Model$, $\mathit{AF}$, $\mathit{Cm}$, $\mathit{Dep}$) of the dependency map. Operation arguments $\mathit{args}$ are operation-specific (e.g., $\textsc{Hypothesize}$ takes the candidate $\gamma$ and its supporting records $\mathit{deps}$). Preconditions are listed in \Cref{tab:ops-full}.}
\label{alg:apply}
\begin{algorithmic}[1]
\Require $\mathcal{D}_t = (\Model_t, \mathit{AF}_t, \mathit{Cm}_t, \mathit{Dep}_t)$, speaker $i \in \mathit{Ags}$
\State $(\Model', \mathit{AF}', \mathit{Cm}', \mathit{Dep}') \gets (\Model_t, \mathit{AF}_t, \mathit{Cm}_t, \mathit{Dep}_t)$
\Switch{$op$}
  \Case{$\textsc{Observe}(\psi,\,\iota)$}
    \State $\Model' \gets \Model_t \announce{!\psi}$
    \State $\alpha \gets \mathsf{recordArg}(\iota,\,\mathrm{claim}{=}\psi,\,\mathrm{src}{=}i)$;\ $\mathit{Args}' \gets \mathit{Args}_t \cup \{\alpha\}$
    \State $\mathit{st}'(\alpha) \gets \mathsf{active}$;\quad $\mathit{Cm}'(i) \gets \mathit{Cm}_t(i) \cup \{\alpha\}$;\quad $\mathit{Dep}'(\alpha) \gets \emptyset$
  \EndCase
  \Case{$\textsc{Hypothesize}(\gamma,\, \mathit{deps},\,\iota)$} \Comment{$\mathit{deps} \subseteq \mathit{Args}_t$ records supports for $\gamma$ (\Cref{def:abduction})}
    \State $\Model' \gets \Model_t \announce{\Uparrow\gamma}$;\quad $\mathit{Prop}' \gets \mathit{Prop}_t \cup \{\gamma\}$
    \State $\alpha \gets \mathsf{recordArg}(\iota,\,\mathrm{claim}{=}\gamma,\,\mathrm{src}{=}i)$;\ $\mathit{Args}' \gets \mathit{Args}_t \cup \{\alpha\}$
    \State $\mathit{st}'(\alpha) \gets \mathsf{active}$;\quad $\mathit{Cm}'(i) \gets \mathit{Cm}_t(i) \cup \{\alpha\}$;\quad $\mathit{Dep}'(\alpha) \gets \mathit{deps}$
  \EndCase
  \Case{$\textsc{Support}(\gamma,\, e)$}
    \State $\Model' \gets \Model_t \announce{\uparrow\gamma}$ if $e$ is generic, else $\Model_t \announce{\Uparrow\gamma}$
    \State $\mathit{Dep}'(\alpha_\gamma) \gets \mathit{Dep}_t(\alpha_\gamma) \cup \{\alpha_e\}$
  \EndCase
  \Case{$\textsc{Undermine}(\gamma,\, e)$}
    \State $\Model' \gets \Model_t \announce{\Uparrow\lnot\gamma}$ if $e$ separates $\lnot\gamma$ from $\gamma$, else partial elimination (\Cref{tab:ops-full})
    \State $\mathit{Att}' \gets \mathit{Att}_t \cup \{(\alpha_e,\alpha_\gamma)\}$
    \If{$\mathit{st}_t(\alpha_\gamma)=\mathsf{active}$}
      \State $\mathit{st}'(\alpha_\gamma)\gets\mathsf{weakened}$
    \EndIf
  \EndCase
  \Case{$\textsc{Revise}(\gamma)$}
    \State $\Model' \gets \Model_t \announce{!\lnot\gamma}$ if $\gamma$ is explicitly negated;\quad $\mathit{Att}'$ updated per \Cref{tab:ops-full}
    \State $\mathit{st}'(\alpha_\gamma) \gets \mathsf{abandoned}$
  \EndCase
  \Case{$\textsc{Expand-Awareness}(p)$}
    \State $\mathit{Prop}' \gets \mathit{Prop}_t \cup \{p\}$;\quad $\mathcal{A}'_i \gets \mathcal{A}_i \cup \{p\}$;\quad refine $W$ on $p$ (\Cref{def:awareness})
  \EndCase
  \Case{$\textsc{Resolve}(\gamma,\, \mathit{subsumes})$} \Comment{$\alpha_\gamma$ existing; $\mathit{subsumes} \subseteq \mathit{Args}_t$ optional}
    \State $\Model' \gets \Model_t \announce{!\gamma}$ if consensual, else commit $\alpha_{\mathit{decide}}$ with dissent recorded
    \State $\mathit{st}'(\alpha_\gamma) \gets \mathsf{resolved}$
    \For{$\beta \in \mathit{subsumes}$} \Comment{record the supporting argument $\alpha_\gamma$}
      \State $\mathit{Dep}'(\beta) \gets \mathit{Dep}_t(\beta) \cup \{\alpha_\gamma\}$
    \EndFor
  \EndCase
  \Case{$\textsc{Question}(\chi)$}
    \State $\Model' \gets \Model_t$;\ enqueue $(\B{i}, \chi)$ on the abductive-problem queue (\Cref{def:abduction})
  \EndCase
\EndSwitch
\State \Return $\mathcal{D}_{t+1} \gets (\Model', \mathit{AF}', \mathit{Cm}', \mathit{Dep}')$
\end{algorithmic}
\end{algorithm}

\noindent\textbf{State composition.} Each operation contributes an intermediate state in the per-turn composition. Record identifiers, standing labels, and dependencies in that state determine the inputs to subsequent operations. \textsc{Resolve} changes an existing record's standing without replacing its support links. The query guarantees concern the resulting map: grounding follows its current support links and labels, and retraction reports reverse reachability. Queries leave that map unchanged, and \Cref{lem:standing} keeps its active set conflict-free.

\noindent\textbf{Subsumption at \textsc{Resolve}.} The optional $\mathit{subsumes}$ argument and its for-loop (\Cref{alg:apply}, \textsc{Resolve} case) record a cross-hypothesis dependency at resolution time. When one hypothesis subsumes others, as $h_4$ unifies $h_1$ and $h_3$ at Phase~2 T13, the edge is added at \textsc{Resolve}, because it does not exist when the subsumed hypotheses are created.

\clearpage
\FloatBarrier
\subsection{Full Epistemic-Operation Table}
\label{app:ops-full}

\Cref{tab:ops-full} gives the preconditions and DEL realisations of the operations in \Cref{subsec:operations}. The attack additions of \textsc{Undermine} and \textsc{Revise} are framework updates in the sense of the dynamic logic of abstract argumentation \citep{doutre2014del}. The consensual and authoritative rows describe two readings of \textsc{Resolve}, which share one operation label in the eight-class evaluation.

\begin{table}[!htbp]
\centering
\small
\caption{Full epistemic-operation table: DEL realisations and preconditions.}
\begin{tabular}{@{}p{2.0cm}p{4.4cm}p{5.7cm}@{}}
\toprule
\textbf{Operation} & \textbf{DEL realisation (full)} & \textbf{Precondition} \\
\midrule
\textsc{Observe} & $\announce{!\psi}$ & Sincere assertion; all atomic propositions in $\psi$ are in $\mathit{Prop}_t$. Later-contradicting evidence triggers \textsc{Revise} when the contradiction is explicit. An implicit restatement of the same attribute with a new value is recorded as a further \textsc{Observe}. \\[2pt]
\textsc{Hypothesize} & $\announce{\Uparrow\gamma}$ & Exists surprising $\chi$; $\gamma$ satisfies \Cref{def:abduction}~(1)--(3). \\[2pt]
\textsc{Support} & $\announce{\uparrow\gamma}$, or $\announce{\Uparrow\gamma}$ if the new evidence is $\gamma$-specific & $\gamma \in \mathit{Prop}_t$; new evidence compatible with $\gamma$; not redundant. \\[2pt]
\textsc{Undermine} & $\announce{\Uparrow\lnot\gamma}$ when evidence distinguishes $\lnot\gamma$ from $\gamma$; partial world elimination when a $\gamma$-prediction is falsified & $\alpha_\gamma$ labelled $\mathsf{active}$ or $\mathsf{weakened}$; new evidence reduces $\gamma$'s relative plausibility. \\[2pt]
\textsc{Revise} & $\announce{!\lnot\gamma}$, or addition of an attack edge in $\mathit{Att}$ & Speaker explicitly retracts or contradicts $\gamma$. \\[2pt]
\textsc{Expand-Awareness} & $\mathcal{A}_i \leftarrow \mathcal{A}_i \cup \{p\}$; refine $W$ on $p$ & The utterance contains an atom $p \notin \mathit{Prop}_t$. \\[2pt]
\textsc{Resolve (consensual)} & $\announce{!\gamma}$ (hard public announcement) & $\alpha_\gamma \in S_t \cap \bigcap_i \mathit{Cm}_t(i)$; no incoming attack on $\alpha_\gamma$ is recorded. \\[2pt]
\textsc{Resolve (authoritative)} & $\alpha_\mathit{decide} \in \mathit{Cm}(i_\mathit{authority})$; dissenting commitments recorded; no hard announcement of $\gamma$ & $\alpha_\gamma \in S_t$ with no recorded incoming attack; the speaker has decision authority and dissent is recorded. \\[2pt]
\textsc{Question} & no DEL update; add $(\B{i}, \chi)$ to the abductive-problem queue & Utterance is interrogative or flags an unexplained observation. \\
\bottomrule
\end{tabular}
\label{tab:ops-full}
\end{table}

%==================================================================
\FloatBarrier
\section{Running Examples}
\label{app:examples}
%==================================================================

\FloatBarrier
\subsection{Phase 1: Muddy Children}
\label{app:phase1}

\label{ex:muddy}
\Cref{fig:muddy} traces the model through the three rounds of the muddy children puzzle (models $\Model_0$--$\Model_3$). The per-turn natural-language transcript follows.

\begin{figure}[!htbp]
\centering
\begin{tikzpicture}[
  world/.style={circle, draw, inner sep=0.3pt, minimum size=4.7mm, font=\tiny},
  actual/.style={circle, draw, double, double distance=0.5pt, inner sep=0.3pt, minimum size=4.7mm, font=\tiny, fill=blue!10},
  ptitle/.style={font=\scriptsize\bfseries, anchor=south},
  every node/.style={font=\scriptsize}
]

\def\colsep{0.75}
\def\rowsep{0.45}
\def\paneltop{0}
\def\midrow{-0.85}

\def\xa{0}
\def\xb{4.0}
\def\xc{7.7}
\def\xd{11.4}

% Panel 0
\node[ptitle] at (\xa+0.5*\colsep, \paneltop+0.4) {$\Model_0$: $|W_0|=8$};
\foreach \pos/\lab in {{(\xa,\paneltop)}/CCC, {(\xa+\colsep,\paneltop)}/CCM, {(\xa,\paneltop-\rowsep)}/CMC, {(\xa+\colsep,\paneltop-\rowsep)}/CMM, {(\xa,\paneltop-2*\rowsep)}/MCC, {(\xa+\colsep,\paneltop-2*\rowsep)}/MCM, {(\xa+\colsep,\paneltop-3*\rowsep)}/MMM}
  \node[world] at \pos {\lab};
\node[actual] (mmc0) at (\xa,\paneltop-3*\rowsep) {MMC};
\node[below=0pt of mmc0, font=\tiny] {$w^*$};

% Panel 1
\node[ptitle] at (\xb+0.5*\colsep, \paneltop+0.4) {$\Model_1$: $|W_1|=7$};
\foreach \pos/\lab in {{(\xb+\colsep,\paneltop)}/CCM, {(\xb,\paneltop-\rowsep)}/CMC, {(\xb+\colsep,\paneltop-\rowsep)}/CMM, {(\xb,\paneltop-2*\rowsep)}/MCC, {(\xb+\colsep,\paneltop-2*\rowsep)}/MCM, {(\xb+\colsep,\paneltop-3*\rowsep)}/MMM}
  \node[world] at \pos {\lab};
\node[actual] at (\xb,\paneltop-3*\rowsep) {MMC};

% Panel 2
\node[ptitle] at (\xc+0.5*\colsep, \paneltop+0.4) {$\Model_2$: $|W_2|=4$};
\foreach \pos/\lab in {{(\xc+\colsep,\paneltop-\rowsep)}/CMM, {(\xc+\colsep,\paneltop-2*\rowsep)}/MCM, {(\xc+\colsep,\paneltop-3*\rowsep)}/MMM}
  \node[world] at \pos {\lab};
\node[actual] at (\xc,\paneltop-3*\rowsep) {MMC};

% Panel 3
\node[ptitle] at (\xd+0.5*\colsep, \paneltop+0.4) {$\Model_3$: $|W_3|=1$};
\node[actual] at (\xd,\paneltop-3*\rowsep) {MMC};

% Arrows
\draw[-{Stealth[length=2.5mm]}, thick] (\xa+\colsep+0.4,\midrow) -- (\xb-0.4,\midrow);
\node[font=\tiny, anchor=south] at ({(\xa+\colsep+\xb)/2}, \midrow+0.05) {$\announce{!(m_a\!\lor\!m_b\!\lor\!m_c)}$};

\draw[-{Stealth[length=2.5mm]}, thick] (\xb+\colsep+0.4,\midrow) -- (\xc-0.4,\midrow);
\node[font=\tiny, align=center, anchor=south] at ({(\xb+\colsep+\xc)/2}, \midrow+0.05) {3 $\times$ \textsc{Observe}\\[-1pt]\itshape ``I don't know''};

\draw[-{Stealth[length=2.5mm]}, thick] (\xc+\colsep+0.4,\midrow) -- (\xd-0.4,\midrow);
\node[font=\tiny, align=center, anchor=south] at ({(\xc+\colsep+\xd)/2}, \midrow+0.05) {3 $\times$ \textsc{Resolve}\\[-1pt]\itshape ``I'm muddy/clean''};

\end{tikzpicture}
\caption{Epistemic-model (DEL) trajectory through the muddy children puzzle: the world set $W_t$ after each round ($8$, $7$, $4$, $1$ worlds). The actual world $w^*{=}MMC$ is double-circled. Eliminated worlds are dropped and live worlds keep their positions.}
\label{fig:muddy}
\end{figure}

%==================================================================

We present the muddy children puzzle as a conversation with the full epistemic model traced at each turn. Three children, Alice ($a$), Bob ($b$), Carol ($c$), have been playing outside. Alice and Bob are muddy, and Carol is clean. Each child can see the others' foreheads but not their own.

\paragraph{Propositions.} $m_a$, $m_b$, $m_c$ (``child $x$ is muddy''). The actual world is $w^* = (m_a, m_b, \lnot m_c)$, abbreviated $\mathit{MMC}$.

\paragraph{Initial model $\Model_0$.} $W_0 = \{MMM, MMC, MCM, MCC, CMM, CMC, CCM, CCC\}$. Each child $i$ has $w \sim_i w'$ iff $w, w'$ agree on $m_j$ for all $j \neq i$. For Alice (sees Bob=M, Carol=C): $\sim_a$ groups $\{MMC, CMC\}$, $\{MMM, CMM\}$, $\{MCM, CCM\}$, $\{MCC, CCC\}$.

\medskip
\begin{conversation}
\turn{Turn 0: Father}{Children, I can see that at least one of you has mud on your forehead.}
\end{conversation}

\noindent\textbf{Operation:} \textsc{Observe} (public announcement $\psi_0 = m_a \lor m_b \lor m_c$).\\
\textbf{Update:} Eliminate $CCC$. \quad \textbf{Model $\Model_1$:} $|W_1|=7$.\\
\textbf{Common knowledge:} $\Cknow{abc}(m_a \lor m_b \lor m_c)$.\\
\textbf{Individual knowledge:} No child knows their own status.

\medskip
\begin{conversation}
\turn{Turn 1a: Alice}{I don't know.}
\turn{Turn 1b: Bob}{I don't know.}
\turn{Turn 1c: Carol}{I don't know.}
\end{conversation}

\noindent\textbf{Operations:} Three \textsc{Observe} (``I don't know'' is a public announcement).\\
\textbf{Reasoning:} If any child saw two clean faces, they would know they are muddy (since at least one must be). Nobody knew, so all single-muddy worlds are eliminated.\\
\textbf{Model $\Model_2$:} $W_2 = \{MMM, MMC, MCM, CMM\}$, $|W_2|=4$.\\
\textbf{Common knowledge:} At least two children are muddy.

\noindent\textbf{Individual knowledge after $\Model_2$:}
\begin{itemize}[leftmargin=*, nosep]
\item Alice sees Bob=M, Carol=C. Consistent worlds in $W_2$: $\{MMC\}$. \emph{Alice knows she is muddy.}
\item Bob sees Alice=M, Carol=C. Consistent worlds in $W_2$: $\{MMC\}$. \emph{Bob knows he is muddy.}
\item Carol sees Alice=M, Bob=M. Consistent worlds in $W_2$: $\{MMC, MMM\}$. Carol doesn't know yet, but can deduce from Alice's and Bob's round-2 answers.
\end{itemize}

\medskip
\begin{conversation}
\turn{Turn 2a: Alice}{Yes, I'm muddy.}
\turn{Turn 2b: Bob}{Yes, I'm muddy.}
\turn{Turn 2c: Carol}{Yes, I'm clean.}
\end{conversation}

\noindent\textbf{Operations:} Three \textsc{Resolve}.\\
\textbf{Model $\Model_3$:} $W_3 = \{MMC\}$. Complete knowledge. $\Cknow{abc}(m_a \land m_b \land \lnot m_c)$.

\paragraph{Verification.} At each step, the system's model must match $\Model_0 \to \Model_1 \to \Model_2 \to \Model_3$. The Interpreter must label Turn~0 and Turns~1a--c as \textsc{Observe} and Turns~2a--c as \textsc{Resolve}.

\FloatBarrier
\subsection{Phase 2: System Debugging}
\label{app:phase2}

Three engineers debug a cascading failure. They are Alice (backend engineer; can see Auth Service and Payment Service application logs), Bob (infrastructure engineer; can see Payment Service metrics, Redis metrics, and external service dashboards) and Carol (on-call SRE; can see the alerting dashboard and customer complaint tickets but needs Alice or Bob to dig into specifics).

\paragraph{Ground truth (known to us, not to the engineers).} Yesterday's deployment introduced a bug in Auth Service: the token refresh path computes expiry from Unix epoch instead of current time, causing refreshed tokens to be immediately expired. The frontend retries expired tokens, generating 3$\times$ normal traffic to Auth. Every auth request hits the shared Redis cluster for session lookup, so the retry storm exhausts Redis's connection pool. Payment Service uses the same Redis cluster for rate-limit counters. With Redis down, rate-limit checks fail and requests go to Stripe unthrottled, triggering Stripe's 429 rate-limit responses. Separately, a stale monitoring rule maps Redis connection errors to ``database health degradation'' alerts, creating a red herring.

\paragraph{What makes this Phase 2.} Unlike Phase~1, the hypothesis space is not given upfront: hypotheses are \emph{constructed} during the conversation through abductive reasoning. The conversation takes a wrong turn (initially attributing auth failures to Redis), evidence requires domain-specific interpretation (error code types), and a causal direction reversal is the key epistemic shift. The monitoring miscategorisation tests awareness expansion.

\paragraph{Formal model specification.} We define the symbolic model that the engine maintains alongside the conversation. Unlike Phase~1 where the world set is fixed at $2^k$ from the start, Phase~2 extends the model dynamically as hypotheses are generated and awareness expands. The abductive mechanism implements \Cref{def:abduction}: the symbolic engine detects surprising observations ($\Model, w \not\models \B{i}\chi$), the LLM generates candidate hypotheses $\gamma$, and the engine verifies and integrates them via plausibility upgrade.

\emph{Agents:} $\mathit{Ags} = \{a, b, c\}$ (Alice, Bob, Carol).

\emph{Evidence records:} the evidence of each \textsc{Support} or \textsc{Undermine} is registered as a record $\alpha_e$. At T7 the evidence is $o_9$, so $\mathit{Att}_7 = \{(o_9, h_2)\}$.

\emph{Propositions:} the observable propositions are $p_\mathit{auth}, p_\mathit{pay}, p_\mathit{db}, p_\mathit{tok}, p_\mathit{traffic}, p_\mathit{stripe429}, p_\mathit{redis}$ (auth failures, payment failures, database alert, token errors, anomalous traffic, Stripe 429s, Redis issues). These monitored signals form $\mathit{Prop}_0$. T1--T3 announce which of them hold.

\emph{Awareness:} all three engineers hear every utterance, so each proposition in $\mathit{Prop}_t$ is in every $\mathcal{A}_i$, and each \textsc{Expand-Awareness} is applied to every agent's awareness set. Crucially, $\mathit{mis\_monitor}$ (monitoring miscategorises Redis as ``database'') is \emph{not} in any $\mathcal{A}_i$ until T4.

\emph{Hypothesis propositions}, together with any new atoms they name, are added to $\mathit{Prop}$ as they are generated:
\begin{itemize}[leftmargin=*, nosep]
\item T5: $h_1 \coloneqq (p_\mathit{redis} \to \mathit{rate\_bypass} \to p_\mathit{stripe429})$ added to $\mathit{Prop}$.
\item T6: $h_2 \coloneqq (p_\mathit{redis} \to p_\mathit{auth})$ added.
\item T11: $h_3 \coloneqq (\mathit{tok\_bug} \to \mathit{retry\_loop} \to p_\mathit{traffic})$ added.
\item T12: $h_4 \coloneqq (\mathit{tok\_bug} \to \mathit{retry\_loop} \to p_\mathit{redis})$ added.
\end{itemize}

\emph{Plausibility-model trajectory and awareness expansion at T4} are tracked in the per-turn walkthrough below. The key shifts are that $h_1$-worlds are upgraded after T5, $h_2$-worlds are downgraded after T7 (\textsc{Undermine}) and eliminated at T9 (\textsc{Revise}), $\mathit{mis\_monitor}$ enters $\mathcal{A}_i$ at T4, and after T13 only $h_4$-worlds remain, with $h_1$ and $h_3$ linked to $h_4$ in the dependency map.

\bigskip
\begin{conversation}
\turn{T1: Carol}{P1 incident. Three alerts firing: auth failure rate up, payment failure rate up, database health degradation. Customer complaints started around 2:15am.}
\end{conversation}

\noindent\textbf{Operations:} \textsc{Observe}~$\times 4$ (three alerts and the onset of complaints), \textsc{Question} (implicit: ``what is causing the failures?'').

\noindent\textbf{Why these classifications:} Carol is reporting factual observations from her dashboard, not proposing explanations. Each alert is a separate observation. The implicit question is what opens the deliberation.

\noindent\textbf{Model state:}
\begin{itemize}[leftmargin=*, nosep]
\item Observations: $o_1$: auth failure rate elevated; $o_2$: payment failure rate elevated; $o_3$: ``database health degradation'' alert; $o_4$: customer complaints from $\sim$2:15am.
\item Hypotheses: none yet.
\item Open questions: What is causing the three alert types? Are they related or independent?
\item Causal chain: unknown.
\end{itemize}

\noindent\textbf{Symbolic model $\Model_1$:}
$\mathit{Prop}_1 = \mathit{Prop}_0$.
Public announcements of $o_1, o_2, o_3, o_4$ eliminate worlds inconsistent with these observations. $o_4$ restates $p_\mathit{auth} \lor p_\mathit{pay}$ from the customer side.
$\mathcal{A}_i = \mathit{Prop}_1$ for all $i$. Note: $\mathit{mis\_monitor} \notin \mathcal{A}_i$ (no agent is yet aware that ``database'' alert may be mislabelled).
No hypotheses in model; no plausibility distinctions beyond observations. All agents share the same epistemic state: $\K{a}\K{b}\K{c}(o_1 \land o_2 \land o_3 \land o_4)$.

\medskip
\begin{conversation}
\turn{T2: Alice}{I'm seeing 401 Unauthorized spikes in auth logs starting around 2am. Tokens being rejected as expired. But I also see our auth request volume is way up, about 3$\times$ normal. That's strange. More users shouldn't be logging in at 2am.}
\end{conversation}

\noindent\textbf{Operations:} \textsc{Observe} (two new facts), \textsc{Question} (``why is traffic 3$\times$?'').

\noindent\textbf{Why:} Alice reports specific observations from her logs. The 401 error code, the ``token expired'' detail, and the 3$\times$ traffic volume are all factual reports. Her remark that ``more users shouldn't be logging in at 2am'' signals a surprising observation that doesn't yet have an explanation. This is an implicit abductive problem that will drive later hypothesis generation.

\noindent\textbf{Model state:}
\begin{itemize}[leftmargin=*, nosep]
\item New observations: $o_5$: 401 ``token expired'' errors from $\sim$2am; $o_6$: auth traffic 3$\times$ normal at 2am.
\item New open question: Why is auth traffic 3$\times$ at 2am? (Anomalous, doesn't match normal usage.)
\end{itemize}

\noindent\textbf{Symbolic model $\Model_2$:}
$\mathit{Prop}_2 = \mathit{Prop}_1$.
Public announcements of $o_5, o_6$. Worlds where auth errors are not ``token expired'' are eliminated; worlds where traffic is normal are eliminated.
$\Model_2, w^* \models \K{a}(o_5 \land o_6)$; after public announcement, $\Cknow{abc}(o_5 \land o_6)$.
\emph{Abductive problem detected (surprise check):} just before $o_6$ is announced, $w^* \models o_6$ but no agent believes $o_6$, since nothing in the model predicts elevated traffic at 2am. So $o_6$ poses an abductive problem in the sense of \Cref{def:abduction} \citep{velazquez2013abductive}. The announcement makes $o_6$ common knowledge but leaves it unexplained. No candidate $\gamma$ has been proposed yet, so the problem remains open until T11.

\medskip
\begin{conversation}
\turn{T3: Bob}{Payment side, I see Stripe returning 429s. We're sending them way more requests than normal. And I can confirm I see the database alert too, Redis connection timeouts from Payment Service.}
\end{conversation}

\noindent\textbf{Operations:} \textsc{Observe}~$\times 2$ (Stripe 429s with elevated request rate, Redis timeouts).

\noindent\textbf{Why:} Pure observation from Bob's infrastructure perspective. The Redis timeouts are particularly important because they will later become central to the causal reasoning. At this point, nobody has proposed a hypothesis yet. The group is still gathering data.

\noindent\textbf{Model state:}
\begin{itemize}[leftmargin=*, nosep]
\item New observations: $o_7$: Stripe 429 (rate-limit) errors; $o_8$: Redis connection timeouts from Payment side.
\item Note: We now have observations from all three engineers. The picture so far: auth tokens failing, payment requests failing, Redis having connection issues. The question is how these relate.
\end{itemize}

\noindent\textbf{Symbolic model $\Model_3$:}
$\mathit{Prop}_3 = \mathit{Prop}_2$.
Public announcements of $o_7, o_8$. $\Cknow{abc}(o_7 \land o_8)$.
No hypotheses yet. The model contains 8 observations but no causal structure. The world set $W_3$ still contains all worlds consistent with $o_1$--$o_8$, including worlds where the symptoms are independent, worlds where Redis causes everything, worlds where auth causes everything, etc. No plausibility distinctions among these.

\medskip
\begin{conversation}
\turn{T4: Carol}{OK, so the database alert is about Redis, not the primary DB. Still concerning. Bob, are you seeing Redis issues from Payment's side?}
\end{conversation}

\noindent\textbf{Operations:} \textsc{Expand-Awareness} (the ``database'' alert is really a Redis alert), \textsc{Question}.

\noindent\textbf{Why \textsc{Expand-Awareness}:} This is the first moment where a proposition that was previously outside the group's reasoning enters the conversation. Everyone initially took the ``database health degradation'' alert at face value and reasoned about a database problem. Carol's reclassification introduces the proposition ``the monitoring rule conflates Redis with database,'' which was not previously in anyone's awareness set. In the formal framework, this adds new formulas to $\mathcal{A}_i$ for all agents, expanding the set of worlds they can distinguish.

\noindent\textbf{Model state:}
\begin{itemize}[leftmargin=*, nosep]
\item Observation $o_3$ reclassified: alert labelled ``database'' is actually about Redis connection errors.
\item Awareness expansion: the proposition $\mathit{mis\_monitor}$ has entered the group's awareness. Before this turn, the group was reasoning about a possible primary database problem. That possibility is now distinguishable from a mislabelled Redis alert.
\end{itemize}

\noindent\textbf{Symbolic model $\Model_4$:}
\emph{Awareness expansion:} $\mathit{mis\_monitor} \coloneqq$ ``monitoring conflates Redis with database.'' Before T4: $\mathit{mis\_monitor} \notin \mathcal{A}_i$ for all $i$. After T4: $\mathit{mis\_monitor} \in \mathcal{A}_i$ for all $i$.
$\mathit{Prop}_4 = \mathit{Prop}_3 \cup \{\mathit{mis\_monitor}\}$.
\emph{World space restructuring:} Before expansion, agents could not distinguish $w_\mathit{db\_problem}$ (primary DB failing) from $w_\mathit{redis\_mislabel}$ (Redis failing, mislabelled as DB). After expansion, these are distinct worlds, and Bob's \textsc{Observe} at T5 (connection pool exhausted) supports the Redis reading.
\emph{Effect:} $o_3$ can now be read either as a primary-database problem or as a mislabelled Redis alert.

\medskip
\begin{conversation}
\turn{T5: Bob}{Yes. Connection pool exhausted. We can't get connections to Redis. That's why our rate-limit checks are failing, we store rate-limit counters in Redis. When we can't check the counter, our code falls through and sends the request to Stripe anyway. That would explain the 429s, we're hitting Stripe without rate limiting.}
\end{conversation}

\noindent\textbf{Operations:} \textsc{Observe} (pool exhaustion details), \textsc{Hypothesize} ($h_1$).

\noindent\textbf{Why \textsc{Hypothesize}:} Bob doesn't just report facts: he constructs an \emph{explanatory chain}: Redis pool exhausted $\to$ rate-limit checks fail $\to$ requests go to Stripe unthrottled $\to$ 429s. This is an abductive step: Bob has a surprising observation ($o_7$: Stripe 429s) and proposes a hypothesis ($h_1$) that explains it. In the formal framework, this is an abductive solution integrated at the belief level via plausibility upgrade.

\noindent\textbf{Model state:}
\begin{itemize}[leftmargin=*, nosep]
\item New observation $o_{8'}$: Redis connection pool exhausted (claim $p_\mathit{redis}$).
\item $h_1$: Redis pool exhaustion $\to$ rate-limit check failure $\to$ unthrottled Stripe requests $\to$ 429s.\\
  Status: \textbf{active}. Plausibility: \textbf{high}. Supported by: $o_7$, $o_8$.
\item This is the first causal hypothesis. It explains the Stripe failures but does not yet explain the auth failures or the Redis exhaustion itself.
\end{itemize}

\noindent\textbf{Symbolic model $\Model_5$:}
\emph{Abductive cycle for $o_7$:}
\begin{enumerate}[leftmargin=*, nosep]
\item \emph{Surprise check (engine):} just before $o_7$ was announced at T3, $w^* \models o_7$ but Bob did not believe $o_7$, since no hypothesis predicted Stripe rate-limit errors. So $o_7$ poses an abductive problem $(\B{b}, o_7)$.
\item \emph{Hypothesis generation (LLM):} The LLM Interpreter classifies Bob's utterance as \textsc{Hypothesize} and extracts the candidate: $\gamma \coloneqq h_1 \coloneqq (p_\mathit{redis} \to \mathit{rate\_bypass} \to p_\mathit{stripe429})$.
\item \emph{Verification (engine):} (a)~Consistency: $h_1$ is consistent with $\Model_4$ (no contradictions). (b)~Explanatory: after plausibility upgrade with $h_1$, agent would believe $o_7$ (Redis being down plus $h_1$ entails Stripe 429s). (c)~Non-trivial: $h_1 \neq o_7$ and $h_1$ is not already believed. All checks pass.
\item \emph{Integration (engine):} $\mathit{Prop}_5 = \mathit{Prop}_4 \cup \{h_1, \mathit{rate\_bypass}\}$. Plausibility upgrade: $h_1$-worlds become more plausible for all agents (public communication). For all $i$, in $W_5$: $w \models h_1 \implies w \prec_i w'$ for $w' \models \lnot h_1$ (given consistency with $o_7, o_8$).
\end{enumerate}
\emph{Epistemic status:} $\B{i}h_1$ for all $i$ (believed, not known). $\lnot\K{i}h_1$ because alternative explanations for $o_7$ have not been ruled out.
\emph{Observations explained:} $h_1$ explains $o_7$ (Stripe 429s), taking $o_8$ (Redis timeouts) as its premise. Does \textbf{not} explain $o_5$ (token errors) or $o_6$ (3$\times$ traffic). The abductive problem for $o_6$ (flagged at T2) remains open.

\medskip
\begin{conversation}
\turn{T6: Carol}{So the chain might be: Redis is sick $\to$ Payment loses rate limiting $\to$ Stripe gets hammered. And separately, Redis being sick $\to$ Auth has problems too? Alice, does Auth use Redis?}
\end{conversation}

\noindent\textbf{Operations:} \textsc{Support} ($h_1$, adopted by Carol as the working picture), \textsc{Hypothesize} ($h_2$: Redis $\to$ Auth failures), \textsc{Question}.

\noindent\textbf{Why:} Carol synthesises the current understanding and extends it. By proposing $h_2$, she is constructing a unified picture: Redis is the single root cause, and both the payment and auth failures are downstream effects. This is a natural abductive move: explaining all symptoms with one cause is more parsimonious.

\noindent\textbf{Model state:}
\begin{itemize}[leftmargin=*, nosep]
\item $h_2$: Redis failure $\to$ Auth failures (via shared Redis cluster for session caching).\\
  Status: \textbf{active}. Plausibility: \textbf{medium}. Supported by: $o_8$. Basis: if Auth shares Redis, its failures could cascade.
\item Leading picture at this point: \emph{Redis is the single root cause}. Everything downstream.
\item This picture is \textbf{wrong}, and the next turns will overturn it.
\end{itemize}

\noindent\textbf{Symbolic model $\Model_6$:}
\emph{Second abductive step:} $h_2 \coloneqq (p_\mathit{redis} \to p_\mathit{auth})$ answers the abductive problem posed by $o_1$, which no agent believed before its announcement at T1. $\mathit{Prop}_6 = \mathit{Prop}_5 \cup \{h_2\}$.
\emph{Plausibility upgrade:} $h_2$-worlds upgraded to plausible (medium). For all $i$: worlds where $h_1 \land h_2$ both hold (single root cause = Redis) are now the most plausible: $w_{h_1 \land h_2} \prec_i w_{\lnot h_1 \land h_2} \prec_i w_{h_1 \land \lnot h_2} \prec_i w_{\lnot h_1 \land \lnot h_2}$. \textsc{Support} registers Carol's restatement as $\alpha_e \in \mathit{Dep}(h_1)$, and $\announce{\uparrow h_1}$ leaves this order unchanged, since the most plausible $h_1$-world $w_{h_1 \land h_2}$ is already minimal.
\emph{Epistemic status:} $\B{i}h_1$ (high), $\B{i}h_2$ (medium). The ``single root cause'' picture is a belief, not knowledge.
\emph{Observations explained by $h_1 \land h_2$:} $o_1$ (auth failures), $o_2$ (payment failures), $o_7$ (Stripe 429s), $o_8$ (Redis timeouts). Remaining unexplained: $o_6$ (3$\times$ traffic).

\medskip
\begin{conversation}
\turn{T7: Alice}{Yes, Auth uses the same Redis cluster for session caching. If Redis is down, token validation would fail because we can't look up the session. That would explain the 401s... but wait, the 401s I'm seeing are specifically `token expired,' not `session lookup failed.' Those are different error codes. The tokens are being rejected because their expiry timestamps are in the past, not because Redis is unavailable.}
\end{conversation}

\noindent\textbf{Operations:} \textsc{Observe} ($o_9$: error code is ``token expired,'' not ``session lookup failed''), \textsc{Undermine} ($h_2$).

\noindent\textbf{Why \textsc{Undermine} and not just \textsc{Observe}:} Alice initially supports $h_2$ (``If Redis is down, token validation would fail... that would explain the 401s''), but then notices a \emph{discrepancy}: the specific error code is inconsistent with the Redis-causes-auth hypothesis. If Redis were the cause, the error would be a session lookup failure (503), not a token expiry (401). This observation doesn't just add new data. It directly decreases the plausibility of $h_2$. In the formal framework, worlds where $h_2$ is true become less plausible because they predict a different error code than what was observed.

\noindent\textbf{This is the most important diagnostic moment in the conversation.} Alice's mid-utterance correction (``but wait...'') shows real-time epistemic revision: she starts by supporting the hypothesis and then undermines it within the same turn.

\noindent\textbf{Model state:}
\begin{itemize}[leftmargin=*, nosep]
\item $h_2$ status: \textbf{active} $\to$ \textbf{weakened}. The evidence ($o_9$) is inconsistent with the predicted error type.
\item New observation $o_9$: Auth errors are 401 ``token expired,'' not 503 ``service unavailable.''
\item The single-root-cause picture (Redis causes everything) is now under strain.
\end{itemize}

\noindent\textbf{Symbolic model $\Model_7$:}
\emph{Observation:} $o_9$ is a further \textsc{Observe} record with claim $p_\mathit{tok}$: the error is ``token expired'', not a 503.
\emph{Undermine $h_2$:} $h_2$ predicts session-lookup failures rather than token expiry, and $o_9$ makes this contrast explicit, so \textsc{Undermine} applies $\announce{\Uparrow\lnot h_2}$ (\Cref{tab:ops-full}). \emph{Plausibility downgrade:} for all $i$, $h_2$-worlds become less plausible: $w_{\lnot h_2} \prec_i w_{h_2}$ (reversing the T6 upgrade).
\emph{Epistemic status:} $\B{i}\lnot h_2$ (agents now believe $h_2$ is false), but $\lnot\K{i}\lnot h_2$ ($h_2$ is not yet eliminated, Alice hasn't explicitly ruled it out). The ordering is: $w_{h_1 \land \lnot h_2} \prec_i w_{\lnot h_1 \land \lnot h_2} \prec_i w_{h_1 \land h_2} \prec_i w_{\lnot h_1 \land h_2}$.
\emph{Key formal point:} here \textsc{Undermine} downgrades the $h_2$-worlds without eliminating them. The remaining $h_2$-worlds are epistemically possible but no longer believed: knowledge corresponds to world elimination, belief to plausibility reordering.

\medskip
\begin{conversation}
\turn{T8: Carol}{Hmm. So the auth failures might not be caused by Redis after all?}
\turn{T9: Alice}{I don't think so. The error type is wrong. If Redis were down, Auth would return a 503 Service Unavailable, not a 401 with `token expired.' I'm seeing 401s. So the token expiry issue is a different problem from the Redis problem.}
\end{conversation}

\noindent\textbf{Operations (T8):} \textsc{Question} (Carol seeks confirmation of the implication).\\
\textbf{Operations (T9):} \textsc{Revise} (Alice explicitly abandons $h_2$).

\noindent\textbf{Why \textsc{Revise}:} This is a genuine belief revision, not just weakening. Alice explicitly states that the token expiry issue is ``a different problem from the Redis problem.'' This separates what was previously thought to be a single-cause situation into (at least) two independent problems. In the formal framework, the hard announcement $!\lnot h_2$ removes the $h_2$-worlds.

\noindent\textbf{Model state:}
\begin{itemize}[leftmargin=*, nosep]
\item $h_2$: \textbf{weakened} $\to$ \textbf{abandoned}. Auth failures are not caused by Redis.
\item The group's picture has fundamentally changed: there are at least two independent problems, not one.
\item Open question (newly urgent): What \emph{is} causing the token expiry errors?
\end{itemize}

\noindent\textbf{Symbolic model $\Model_9$:}
\emph{Revision:} Alice explicitly negates $h_2$, so \textsc{Revise} applies the hard announcement $!\lnot h_2$ (\Cref{tab:ops-full}), which eliminates every $h_2$-world.
After T9, $\K{i}\lnot h_2$ holds for all $i$. The group now \emph{knows} $\lnot h_2$, not merely believes it.
\emph{Updated plausibility:} $w_{h_1 \land \lnot h_2} \prec_i w_{\lnot h_1 \land \lnot h_2}$. No active link between Auth and Redis remains.
\emph{Observations now unexplained:} With $h_2$ abandoned, $o_1$ (auth failures), $o_5$ (token errors), and $o_6$ (3$\times$ traffic) have no causal explanation, and $o_6$ remains the open abductive problem flagged at T2.
\emph{Key formal distinction from T7:} At T7, $h_2$ was \emph{undermined} (plausibility downgraded but worlds retained). At T9, $h_2$ is \emph{revised} (worlds eliminated). This is the difference between \textsc{Undermine} and \textsc{Revise} in the formal framework.

\medskip
\begin{conversation}
\turn{T10: Bob}{But then why is auth traffic 3$\times$ normal? If the token expiry issue is independent of Redis, what's generating all that auth traffic?}
\end{conversation}

\noindent\textbf{Operations:} \textsc{Question}.

\noindent\textbf{Why this matters:} Bob identifies that the unexplained 3$\times$ traffic ($o_6$) is now a critical clue. If auth failures aren't caused by Redis, then the elevated traffic isn't explained by the current model. This question drives the next abductive step. In the formal framework, this is a recognition of an abductive problem: there exists a surprising observation ($o_6$) that the current set of hypotheses cannot explain.

\noindent\textbf{Symbolic model $\Model_{10}$:}
No structural change to $W$ or plausibility orderings. The \textsc{Question} keeps $o_6$ on the abductive-problem queue: $(\B{i}, o_6) \in \mathit{AbdProb}_{10}$. The model records that $o_6$ is surprising and unexplained.

\medskip
\begin{conversation}
\turn{T11: Alice}{Could be the frontend retrying. When a user gets a `token expired' error, the frontend automatically tries to refresh the token. If the new token is also expired, it retries again. That's a retry loop. So the token bug generates its own amplified traffic.}
\end{conversation}

\noindent\textbf{Operations:} \textsc{Hypothesize} ($h_3$).

\noindent\textbf{Why:} Alice proposes a new hypothesis to explain the anomalous traffic observation. This is a classic abductive step: the surprising fact ($o_6$: 3$\times$ traffic) triggers the generation of an explanatory hypothesis ($h_3$: a feedback loop where the bug amplifies its own traffic). Note that $h_3$ involves domain knowledge about the frontend's retry behaviour, something the group becomes aware of through Alice's contribution.

\noindent\textbf{Model state:}
\begin{itemize}[leftmargin=*, nosep]
\item $h_3$: Token bug $\to$ frontend retry loop $\to$ 3$\times$ traffic amplification.\\
  Status: \textbf{active}. Plausibility: \textbf{medium}. Explains $o_6$.
\item At this point, the group has two separate explanations: $h_1$ (Redis $\to$ Stripe problems) and $h_3$ (token bug $\to$ traffic amplification). They haven't yet connected these.
\end{itemize}

\noindent\textbf{Symbolic model $\Model_{11}$:}
\emph{Abductive cycle for $o_6$ (open since T2):}
\begin{enumerate}[leftmargin=*, nosep]
\item \emph{Surprise check:} $o_6$ (3$\times$ traffic) was flagged at T2 as an observation not predicted by any hypothesis. After $h_2$'s abandonment at T9, $o_6$ still has no explanatory hypothesis in the current model. Abductive problem $(\B{i}, o_6)$ remains open.
\item \emph{Hypothesis generation (LLM):} Alice's utterance classified as \textsc{Hypothesize}. Candidate extracted: $\gamma \coloneqq h_3 \coloneqq (\mathit{tok\_bug} \to \mathit{retry\_loop} \to p_\mathit{traffic})$.
\item \emph{Verification (engine):} (a)~Consistency: $h_3$ is consistent with $\Model_{10}$. (b)~Explanatory: if $\mathit{tok\_bug}$ is true, the retry loop produces elevated traffic, so $h_3 \land \mathit{tok\_bug}$ entails $o_6$. (c)~Non-trivial: $h_3 \neq o_6$. All checks pass.
\item \emph{Integration:} $\mathit{Prop}_{11} = \mathit{Prop}_{10} \cup \{h_3, \mathit{tok\_bug}, \mathit{retry\_loop}\}$. Plausibility upgrade: $h_3$-worlds upgraded. For all $i$: $w_{h_3} \prec_i w_{\lnot h_3}$.
\end{enumerate}
\emph{New atoms:} $\mathit{tok\_bug}$ (a deployment bug causes token expiry) and $\mathit{retry\_loop}$ (frontend retries create amplification) enter $\mathit{Prop}$ with $h_3$ (step 4).
\emph{Epistemic status:} $\B{i}h_3$ (medium plausibility, plausible but unconfirmed). The abductive problem for $o_6$ is now \textbf{closed} (explained by $h_3$).
\emph{Model topology:} Two independent causal chains coexist: $h_1$ (Redis $\to$ Stripe) and $h_3$ (token bug $\to$ traffic). No connection between them yet. $o_6$ is now explained; $o_5$ (token errors) is partially explained.

\medskip
\begin{conversation}
\turn{T12: Carol}{Wait, so the retry storm from the token bug could be what's exhausting Redis? Not Redis causing the auth problem, but the auth problem causing the Redis overload?}
\end{conversation}

\noindent\textbf{Operations:} \textsc{Hypothesize} + \textsc{Revise}. \textbf{This is the critical turn: causal direction reversal.}

\noindent\textbf{Why this is the key epistemic shift:} Carol makes the decisive connection. She realises that $h_3$ (token bug causes retry storm) can be linked to $h_1$ (Redis exhaustion causes Stripe failures) by \emph{reversing the causal direction between Auth and Redis}. The group had previously assumed Redis $\to$ Auth (hypothesis $h_2$, now abandoned). Carol proposes Auth $\to$ Redis: the retry storm is what overwhelms Redis, not the other way around.

\noindent This is both a \textsc{Hypothesize} (introducing the new causal link: retries exhaust Redis) and a \textsc{Revise} of $h_2$: Carol explicitly contradicts Redis $\to$ Auth, so the engine adds the attack $(h_4, h_2)$ and $h_2$ stays abandoned. The result is again a single causal chain, with a structure opposite to the abandoned $h_2$.

\noindent\textbf{Model state:}
\begin{itemize}[leftmargin=*, nosep]
\item $h_4$: Token bug $\to$ retry storm $\to$ Redis connection pool exhaustion.\\
  Status: \textbf{active}. Plausibility: \textbf{medium}. Links $h_3$ to $h_1$.
\item \textbf{Key epistemic shift recorded}: Causal direction between Auth and Redis has reversed. Previously: Redis failure was thought to cause auth problems ($h_2$). Now: auth problems (the token bug and resulting retry storm) are proposed to cause Redis failure ($h_4$).
\item Emerging unified chain: token bug $\to$ retries $\to$ Redis exhaustion $\to$ rate-limit bypass $\to$ Stripe 429s.
\end{itemize}

\noindent\textbf{Symbolic model $\Model_{12}$:}
\emph{Hypothesis:} $h_4 \coloneqq (\mathit{tok\_bug} \to \mathit{retry\_loop} \to p_\mathit{redis})$. It answers the problem posed by $o_8$, which no agent believed before its announcement at T3. Note the causal direction: Auth $\to$ Redis, opposite of the abandoned $h_2$ (Redis $\to$ Auth).
$\mathit{Prop}_{12} = \mathit{Prop}_{11} \cup \{h_4\}$.
\emph{Plausibility upgrade:} $h_4$-worlds upgraded to plausible. The combined hypothesis $h_1 \land h_3 \land h_4$ now forms a unified causal chain and is more plausible than the independent-problems picture.
For all $i$: every $h_4$-world precedes every $\lnot h_4$-world, and $w_{h_1 \land h_3 \land h_4}$ is the most plausible.
\emph{Epistemic status:} $\B{i}h_4$ (believed, not known, awaiting mechanistic confirmation).
\emph{Structural revision:} The dependency map records $h_4$ with $h_3 \in \mathit{Dep}(h_4)$, so the recorded link between Auth and Redis now runs from Auth to Redis. The model records this as a key epistemic shift with provenance: $h_2$ (Redis $\to$ Auth, T6, abandoned T9 due to $o_9$) $\to$ $h_4$ (Auth $\to$ Redis, T12, supported by $h_3$).
\emph{Unification:} If $h_4$ is accepted, then $h_1 \land h_3 \land h_4$ form a single causal chain explaining \emph{all} observations $o_1$--$o_9$ and $o_{8'}$.

\medskip
\begin{conversation}
\turn{T13: Alice}{That's... actually plausible. If auth traffic is 3$\times$, and every auth request hits Redis for session lookup, the Redis connection pool could be overwhelmed. So the chain would be: token bug $\to$ retry storm $\to$ Redis exhaustion. And then Redis exhaustion $\to$ Payment rate-limit bypass $\to$ Stripe 429s. The whole thing cascades from the token bug.}
\end{conversation}

\noindent\textbf{Operations:} \textsc{Support} (confirming $h_4$ with mechanistic reasoning), \textsc{Resolve} (the unified causal chain is accepted).

\noindent\textbf{Why \textsc{Resolve}:} Alice provides the mechanistic argument that makes $h_4$ convincing: 3$\times$ traffic, every request hitting Redis, pool exhaustion is the expected consequence. She then states the complete unified chain explicitly. At this point, the hypothesis is elevated from tentative belief to high-confidence accepted explanation.

\noindent\textbf{Symbolic model $\Model_{13}$ (final):}
\emph{Support + Resolve:} Alice's mechanistic argument provides sufficient evidence to elevate $h_4$ from belief to knowledge. It is registered as $\alpha_e \in \mathit{Dep}(h_4)$.
\textsc{Resolve} announces $!h_4$, which eliminates the $\lnot h_4$-worlds, so $\K{i}h_4$ holds for all $i$ after T13. It also records $h_1$ and $h_3$ as subsumed by $h_4$ (\Cref{app:algorithm}), linking the unified chain $h_\mathit{unified} = h_1 \land h_3 \land h_4$ through the dependency map.
\emph{Final plausibility:} The model converges to a single maximally plausible world class: $W_{13}^\mathit{max} = \{w : w \models h_\mathit{unified}\}$. The $\lnot h_4$-worlds are eliminated.
\emph{Hypothesis lifecycle summary:}
\begin{itemize}[leftmargin=*, nosep]
\item $h_1$: $\emptyset \xrightarrow{\text{T5, Hyp}}$ active. At T13 it gains a dependency link to $h_4$ and joins the unified chain (correct mechanism).
\item $h_2$: $\emptyset \xrightarrow{\text{T6, Hyp}}$ active $\xrightarrow{\text{T7, Und}}$ weakened $\xrightarrow{\text{T9, Rev}}$ abandoned. Disproved by $o_9$.
\item $h_3$: $\emptyset \xrightarrow{\text{T11, Hyp}}$ active. At T13 it gains a dependency link to $h_4$ and joins the unified chain.
\item $h_4$: $\emptyset \xrightarrow{\text{T12, Hyp}}$ active $\xrightarrow{\text{T13, Sup+Res}}$ resolved. Key epistemic shift: causal reversal.
\end{itemize}
\emph{Awareness set (final):} $\mathcal{A}_i = \mathit{Prop}_{13}$ for all $i$: the observable propositions, $\mathit{mis\_monitor}$, $h_1$--$h_4$ and their atoms.
\emph{Knowledge (final):} $\Cknow{abc}(h_4 \land \lnot h_2)$, and $h_\mathit{unified}$ holds in every most plausible world.

\paragraph{Final model state.} The unified causal chain (matching the ground truth of \Cref{app:phase2}'s opening): \emph{deployment token bug $\to$ expired tokens $\to$ frontend retries $\to$ 3$\times$ auth traffic $\to$ Redis pool exhaustion $\to$ Payment rate-limit bypass $\to$ unthrottled Stripe requests $\to$ 429s}, with the ``database health'' alert separately flagged at T4 as a possible monitoring miscategorisation. Hypothesis lifecycles are listed above. The per-turn walkthrough records, at each step, the formal grounds (operation, plausibility shift, dependency edge) for the system's tracking, including the two key epistemic shifts at T7--T9 (error-code evidence kills $h_2$) and T12 (causal direction reversal).

\FloatBarrier
\subsection{Phase 3: Architecture Deliberation}
\label{app:phase3}

A team decides how to add real-time collaboration to their document editor. Its members are Alice (product lead; feature scope and timeline), Bob (senior backend engineer; architecture and long-term maintainability) and Carol (frontend engineer; editor integration and user experience).

\paragraph{What makes this Phase 3.} Unlike Phases~1--2, there is no objectively correct answer. The conversation produces a \emph{decision} through deliberation, not a \emph{discovery} of pre-existing fact. The core epistemic primitive shifts from knowledge/belief to \emph{commitments} (public, retractable assertions about what the team will do). The conversation involves genuine trade-offs, persistent disagreement, and a decision made under authority that explicitly records dissent and conditions for re-evaluation.

\paragraph{Model notation.} We track the model using IBIS-style elements: \emph{Issues} ($I_n$, questions under deliberation), \emph{Positions} ($P_n$, proposed answers), \emph{Arguments} (pro/con, attributed to speakers), and \emph{Decisions} (resolved issues with provenance). We also track \emph{Assumptions} ($a_n$): premises that decisions rest on, retractable when new evidence undermines them.

\paragraph{Classification results (cross-model).} \Cref{tab:phase3-results} reports per-model F1 and EM on Phase 3, under the Definitions condition only (matching the cross-scenario comparison in \Cref{tab:cross-model}). Phase 3 reaches comparable or higher accuracy than Phase 2 on all four models, despite the longer 19-turn structure and the epistemic-to-deliberative shift (\Cref{sec:system}).

\begin{table}[!htbp]
\centering
\caption{Phase~3 (deliberation, 19 turns) classification, Definitions condition only. EM is exact match in percent, over 5 runs per cell.}
\label{tab:phase3-results}
\small
\begin{tabular}{lcccc}
\toprule
& \multicolumn{2}{c}{\textbf{Overall}} & \multicolumn{2}{c}{\textbf{Key shifts (4 turns)}} \\
\cmidrule(lr){2-3} \cmidrule(lr){4-5}
Interpreter & F1 & EM & F1 & EM \\
\midrule
Claude Sonnet 4 & \textbf{0.85} & 63 & \textbf{0.92} & \textbf{75} \\
GPT-4o & 0.80 & \textbf{65} & 0.83 & 50 \\
Qwen2.5-7B & 0.75 & 61 & \textbf{0.92} & \textbf{75} \\
Llama-3.1-8B & 0.66 & 33 & 0.59 & 20 \\
\bottomrule
\end{tabular}
\end{table}

\paragraph{Formal model specification.} The symbolic engine maintains a dependency map $(\Model, \mathit{AF}, \mathit{Cm}, \mathit{Dep})$ as defined in \Cref{def:eam}. Phase-3-specific instantiation:

\emph{Agents:} $\mathit{Ags} = \{a, b, c\}$ (Alice, Bob, Carol), with roles: $a$ = product lead (decision authority), $b$ = senior backend (architecture), $c$ = frontend (editor integration).

\emph{Arguments:} each position $P_n$ is itself a record in $\mathit{Args}$ with claim ``adopt $P_n$'', so attacks on positions are attacks between records. The evidence of each \textsc{Support} or \textsc{Undermine} is registered as a new record $\alpha_e$ under the identifier the Interpreter supplies (\Cref{app:algorithm}), committed by the speaker, with the premises the utterance cites as its dependencies. The displayed records carry tuples $(\mathit{claim}, \mathit{speaker}, \mathit{turn}, \mathit{type})$. Within $\mathit{Dep}$ and $\mathit{Affected}^{*}$ below, premise labels such as $a_i$ and $\mathit{timeline}$ abbreviate the identifiers of premise records, registered, with their propositions in $\mathit{Prop}$, by the first record that cites them. In truth assertions they denote the corresponding propositions. The records that rest on a withdrawn premise $\beta$ are reported by $\mathit{Affected}^{*}(\beta)$ (\Cref{def:affected}). Each \textsc{Hypothesize} adds its position to $\mathit{Prop}$, and the decision record of an authoritative \textsc{Resolve} lists the premises it rests on.

\emph{Awareness:} $\mathcal{A}_i$ tracks the solution space visible to each agent. All three agents hear every utterance, so each \textsc{Expand-Awareness} is applied to every agent's awareness set. Initially, only ``implement Operational Transformation (OT)'' and ``implement Conflict-free Replicated Data Type (CRDT)'' are in $\mathcal{A}_i$. Library-based approaches enter at T5 via awareness expansion.

\bigskip
\begin{conversation}
\turn{T1: Alice}{We need real-time collaboration. Users should see each other's edits live, like Google Docs. Ship in six weeks. How do we build it?}
\end{conversation}

\noindent\textbf{Operations:} \textsc{Question} (opens root issue), \textsc{Observe} (constraints stated).

\noindent\textbf{Why:} Alice defines the problem and constraints. In IBIS terms, this opens the root issue. The six-week timeline and the real-time requirement (opened as issue $I_1$) are constraints that will shape the evaluation of all positions.

\noindent\textbf{Model:}
\begin{itemize}[leftmargin=*, nosep]
\item Issue $I_1$: ``How to implement real-time collaboration?'' Status: \textbf{open}.
\item Constraints: 6-week deadline; users must see changes in real time.
\end{itemize}

\noindent\textbf{Symbolic model $\mathit{AF}_1$:}
$\mathit{Args}_1 = \{\mathit{timeline}\}$, the premise record of the six-week deadline registered by \textsc{Observe}. $\mathit{Att}_1 = \emptyset$.
$\mathit{Cm}(a) = \{\mathit{timeline}\}$, $\mathit{Cm}(b) = \mathit{Cm}(c) = \emptyset$.
The \textsc{Question} opens issue $I_1$, whose requirement is real-time collaboration. $\mathit{Prop}_0 = \{\mathit{timeline}, P_1, P_2\}$: the deadline and the two approaches the team starts from.

\medskip
\begin{conversation}
\turn{T2: Bob}{Two established approaches: Operational Transformation, that's what Google Docs uses, and CRDTs, which is what Figma and newer tools use. OT needs a central server for coordination. CRDTs are peer-to-peer capable but more complex to implement.}
\turn{T3: Alice}{What about the six-week timeline? Can we ship either one?}
\turn{T4: Bob}{Implementing either from scratch in six weeks is risky. OT's transformation functions are subtle and buggy. CRDTs have complex data structures.}
\turn{T5: Carol}{What about using a library? Yjs is a mature CRDT library. ShareDB for OT.}
\end{conversation}

\noindent\textbf{Operations (T2):} \textsc{Observe}~$\times 2$ (two approaches exist), which registers the position records $P_1$ and $P_2$.\\
\textbf{Operations (T3):} \textsc{Question} (feasibility sub-issue under timeline constraint).\\
\textbf{Operations (T4):} \textsc{Undermine} ($P_1$ and $P_2$: both risky within timeline).\\
\textbf{Operations (T5):} \textsc{Expand-Awareness} + \textsc{Hypothesize} (library approach was not previously considered).

\noindent\textbf{Why T5 is \textsc{Expand-Awareness}:} Before Carol's suggestion, the discussion was framed as ``which algorithm should we implement?'' Carol introduces a new dimension, using existing libraries, that was outside the group's consideration. This is analogous to awareness expansion in the formal framework: a new set of propositions (about specific libraries, their maturity, their integration properties) enters the group's reasoning.

\noindent\textbf{Model:} \emph{First key reframing}: the issue shifts from ``OT vs CRDT algorithm'' to ``Yjs vs ShareDB library.''
\begin{itemize}[leftmargin=*, nosep]
\item $P_1$ (implement OT from scratch), $P_2$ (implement CRDT from scratch): \textbf{set aside} (weakened). Reason: both too risky for the 6-week timeline. The reason is recorded, so these positions are revisited only if the timeline changes.
\item $P_3$ (use Yjs, CRDT library): \textbf{open}.
\item $P_4$ (use ShareDB, OT library): \textbf{open}.
\item Epistemic shift: reframing from algorithm choice to library choice.
\end{itemize}

\noindent\textbf{Symbolic model $\mathit{AF}_5$:}
$\mathit{Args}_5 = \{\mathit{timeline}, P_1, P_2, \alpha_1, \alpha_2, P_3, P_4\}$, where $\alpha_1 = (P_1\text{ risky}, b, T4, \mathit{con})$ and $\alpha_2 = (P_2\text{ risky}, b, T4, \mathit{con})$ are the evidence records of the T4 \textsc{Undermine}, and \textsc{Hypothesize} registers the position records $P_3$ ($\equiv \alpha_3$) and $P_4$ ($\equiv \alpha_4$) at T5.
$\mathit{Att}_5 = \{(\alpha_1, P_1), (\alpha_2, P_2)\}$, so $\mathit{st}(P_1) = \mathit{st}(P_2) = \mathsf{weakened}$.
$\mathit{Cm}(a) = \{\mathit{timeline}\}$, $\mathit{Cm}(b) = \{P_1, P_2, \alpha_1, \alpha_2\}$, $\mathit{Cm}(c) = \{P_3, P_4\}$.
$\mathit{Dep}(\alpha_1) = \mathit{Dep}(\alpha_2) = \{\mathit{timeline}\}$ (if the timeline is relaxed, these arguments lose force).
\emph{Awareness expansion:} \textsc{Expand-Awareness} adds the library atoms $\mathit{Yjs}$ and $\mathit{ShareDB}$ to $\mathit{Prop}$ and to every $\mathcal{A}_i$ at T5, and \textsc{Hypothesize} adds the positions $P_3$ and $P_4$ built on them. Before T5, only $\{P_1, P_2\}$ were in consideration.
\emph{Positions:} $P_1, P_2$ weakened by $\alpha_1, \alpha_2$. $P_3, P_4$ active, no attacks yet.

\medskip
\begin{conversation}
\turn{T6: Carol}{I've prototyped with Yjs before on a side project. The yjs-prosemirror binding, which is our editor, is well documented. I don't know if ShareDB has the same ProseMirror integration.}
\turn{T7: Bob}{Both work with our Node.js backend. But with Yjs we could go serverless or use a central server. More architectural flexibility.}
\end{conversation}

\noindent\textbf{Operations (T6):} \textsc{Support} ($P_3$: ProseMirror integration exists and Carol has experience), \textsc{Undermine} ($P_4$: integration status unknown).\\
\textbf{Operations (T7):} \textsc{Support} ($P_3$: architectural flexibility).

\noindent\textbf{Why these matter:} Carol's prior experience with Yjs is both evidence (she's prototyped it) and a practical argument (less ramp-up time). Her uncertainty about ShareDB's integration is an asymmetry: $P_3$ has a known integration story while $P_4$ has an unknown one. Bob's point about flexibility adds a different dimension of support.

\noindent\textbf{Model:}
\begin{itemize}[leftmargin=*, nosep]
\item $P_3$ arguments pro: ProseMirror binding exists (Carol, from experience); architectural flexibility (Bob).
\item $P_4$ arguments con: ProseMirror integration unclear (Carol).
\item $P_3$ is emerging as the stronger candidate, but concerns haven't been raised yet.
\end{itemize}

\noindent\textbf{Symbolic model $\mathit{AF}_7$:}
New evidence records: $\alpha_5 = (P_3\text{ ProseMirror ok}, c, T6, \mathit{pro})$, $\alpha_6 = (P_4\text{ ProseMirror unclear}, c, T6, \mathit{con})$, $\alpha_7 = (P_3\text{ arch.\ flexibility}, b, T7, \mathit{pro})$.
$\mathit{Dep}(P_3) = \{\alpha_5, \alpha_7\}$. $\mathit{Att}_7 = \mathit{Att}_5 \cup \{(\alpha_6, P_4)\}$, so $P_4$ is weakened.
$\mathit{Cm}(c) = \{P_3, P_4, \alpha_5, \alpha_6\}$, $\mathit{Cm}(b) = \{P_1, P_2, \alpha_1, \alpha_2, \alpha_7\}$.
\emph{Standing:} $P_3$ is active and grounded through $\alpha_5$ and $\alpha_7$, and $P_4$ is weakened by $\alpha_6$. $P_3$ is the stronger candidate.

\medskip
\begin{conversation}
\turn{T8: Bob}{CRDTs have a known problem with document size. The CRDT metadata grows over time and can get large for long-lived documents. Yjs has some GC mechanisms but they're not trivial.}
\turn{T9: Alice}{Is that a problem for our initial launch? Our documents are typically 5--10 pages.}
\turn{T10: Bob}{Probably not for launch. It's a long-term concern. But I want to flag it because switching from CRDT to OT later would be a rewrite, not a refactor.}
\end{conversation}

\noindent\textbf{Operations (T8):} \textsc{Undermine} ($P_3$: document size risk).\\
\textbf{Operations (T9):} \textsc{Question} (is the risk relevant to our context?).\\
\textbf{Operations (T10):} \textsc{Support} (low severity for launch) + \textsc{Undermine} (irreversibility makes it a strategic risk).

\noindent\textbf{Why this sequence matters:} Bob raises a concern, Alice challenges its relevance to the immediate context, and Bob concedes on the short term but flags the long-term irreversibility. This creates a \emph{conditional risk}, something that's acceptable under current assumptions but becomes problematic if assumptions change. The model must capture not just ``Bob has a concern'' but the precise conditions under which the concern applies.

\noindent\textbf{Model:}
\begin{itemize}[leftmargin=*, nosep]
\item $P_3$ argument con: document size risk (Bob). Severity assessment: \textbf{low for current use case}, \textbf{high if long-lived documents needed}. Risk characteristic: \textbf{irreversible}, since switching later is a rewrite.
\item Sub-issue $I_2$: ``Is document size a problem for us?'' Status: \textbf{provisionally settled}, not for launch.
\end{itemize}

\noindent\textbf{Symbolic model $\mathit{AF}_{10}$:}
New evidence records: $\alpha_8 = (P_3\text{ doc size risk}, b, T8, \mathit{con})$, $\alpha_9 = (P_3\text{ ok for short docs}, b, T10, \mathit{pro})$, $\alpha_{10} = (P_3\text{ irreversible if wrong}, b, T10, \mathit{con})$.
$\mathit{Att}_{10} = \mathit{Att}_7 \cup \{(\alpha_8, P_3), (\alpha_{10}, P_3)\}$, so $P_3$ is weakened from T8.
$\mathit{Dep}(P_3) = \{\alpha_5, \alpha_7, \alpha_9\}$. $\mathit{Dep}(\alpha_9) = \{a_1\}$, where $a_1 \coloneqq$ ``docs are short'', stated by Alice at T9. The assumption explicitly conditions the argument.
$\mathit{Dep}(\alpha_{10}) = \emptyset$ (the irreversibility argument is unconditional).
$\mathit{Cm}_{10}(b) = \mathit{Cm}_7(b) \cup \{\alpha_8, \alpha_9, \alpha_{10}\}$.
\emph{Key:} $\alpha_9$ answers $\alpha_8$ for the launch, conditional on $a_1$, while $\alpha_{10}$ is an independent attack on $P_3$. Sub-issue $I_2$ is provisionally settled: document size is not a launch problem given $a_1$.

\medskip
\begin{conversation}
\turn{T11: Carol}{If we go with Yjs and WebRTC, we could support offline editing natively. User research showed spotty connectivity is a pain point.}
\turn{T12: Bob}{Hmm, but if edits are peer-to-peer, access control is hard. We need role-based permissions.}
\turn{T13: Carol}{Can we use Yjs but with a central server as the sync point? We'd get the CRDT benefits, conflict resolution, offline merge, but the server can enforce access control.}
\turn{T14: Bob}{Yes, that's actually the recommended production setup for Yjs. You run a Yjs WebSocket server as the sync point. And we already run WebSocket servers for notifications.}
\end{conversation}

\noindent\textbf{Operations (T11):} \textsc{Support} ($P_3$: offline editing, grounded in user research).\\
\textbf{Operations (T12):} \textsc{Undermine} ($P_3$ in pure P2P form: access control problem).\\
\textbf{Operations (T13):} \textsc{Hypothesize} (new hybrid position $P_5$: Yjs + central server relay).\\
\textbf{Operations (T14):} \textsc{Support} ($P_5$: standard production setup, fits existing infrastructure).

\noindent\textbf{Why T13 is a key move:} Carol resolves the tension between two competing concerns (CRDT benefits vs.\ access control) by \emph{proposing a hybrid} that keeps the advantages of both. In argumentation terms, she introduces a new position that is not a compromise but a synthesis. In the formal framework, this is hypothesis generation: the hybrid combines options already under discussion into a new position.

\noindent\textbf{Model:} \emph{Second key reframing}: from ``Yjs peer-to-peer'' to ``Yjs server-relayed.''
\begin{itemize}[leftmargin=*, nosep]
\item Position $P_5$: Yjs with central WebSocket server as sync/authority point. Status: \textbf{leading}.
\item $P_5$ pro: CRDT conflict resolution (inherent); offline merge (inherent); access control via server (T13); existing WebSocket infrastructure (T14, Bob); ProseMirror binding (T6, Carol).
\item $P_5$ con: the CRDT document-size concern (Bob, T8), recorded against $P_3$.
\item Pure P2P variant $P_3$: weakened, and now also attacked for access control.
\item Epistemic shift: hybrid position resolves the tension between CRDT benefits and permission requirements.
\end{itemize}

\noindent\textbf{Symbolic model $\mathit{AF}_{14}$:}
New records: the evidence records $\alpha_{11} = (P_3\text{ offline editing}, c, T11, \mathit{pro})$ and $\alpha_{12} = (P_3\text{ P2P access control}, b, T12, \mathit{con})$, the position record $P_5$ ($\equiv \alpha_{13}$, Carol, T13) registered by \textsc{Hypothesize}, and the evidence record $\alpha_{14} = (P_5\text{ std.\ setup + infra}, b, T14, \mathit{pro})$.
$\mathit{Att}_{14} = \mathit{Att}_{10} \cup \{(\alpha_{12}, P_3)\}$.
$\mathit{Dep}(P_3) = \{\alpha_5, \alpha_7, \alpha_9, \alpha_{11}\}$. $\mathit{Dep}(P_5) = \{\alpha_5, \alpha_{11}, \alpha_{14}\}$: the hypothesis rests on the ProseMirror binding and offline merge it shares with $P_3$, and T14 adds the infrastructure fit.
$\mathit{Cm}_{14}(c) = \mathit{Cm}_{10}(c) \cup \{\alpha_{11}, P_5\}$, $\mathit{Cm}_{14}(b) = \mathit{Cm}_{10}(b) \cup \{\alpha_{12}, \alpha_{14}\}$.
\emph{Key:} the server-relayed design answers the access-control concern by construction, so $\alpha_{12}$ attacks only the pure P2P position $P_3$. The document-size concerns $\alpha_8$ and $\alpha_{10}$ are about CRDTs and remain recorded against $P_3$.
\emph{Standing:} $P_5$ is active with no incoming attack, and every record in $\mathit{Dep}^{*}(P_5)$ is active, so a continuation asserting $P_5$ is grounded (\Cref{def:grounded}).

\medskip
\begin{conversation}
\turn{T15: Bob}{I want to come back to the document size issue. If we go CRDT, every edit operation is stored permanently in the CRDT state. For a 10-page document edited for months, the CRDT metadata could be 10--50$\times$ larger than the content. Yjs has compaction but it's not trivial. And switching from CRDT to OT later would be a six-month rewrite.}
\turn{T16: Alice}{How confident are you that the problem will actually manifest? Our documents are short and have burst editing, a few days of activity, then they become read-only.}
\turn{T17: Bob}{For the current use case, probably 80\% chance it's fine. But the Q2 roadmap includes long-running project documents. Those would be edited continuously for months.}
\turn{T18: Alice}{Q2 isn't confirmed. I don't want to make an architectural decision now based on a feature that might not happen. Here's what I propose: we go with Yjs for launch. Bob, write up the risk with specific thresholds, when should we start worrying. If Q2 confirms long-running documents, we evaluate then.}
\turn{T19: Bob}{I'll write it up. But I want it on the record that I think this is short-sighted. If we'd gone with ShareDB, we wouldn't be carrying this risk at all.}
\end{conversation}

\noindent\textbf{Operations (T15):} \textsc{Undermine} ($P_3$: Bob escalates the CRDT document-size concern with specific numbers and the irreversibility argument).\\
\textbf{Operations (T16):} \textsc{Question} (Alice challenges the probability of the risk manifesting).\\
\textbf{Operations (T17):} \textsc{Support} ($P_5$: 80\% fine for current use) + \textsc{Undermine} ($P_3$: the Q2 roadmap would raise the CRDT risk).\\
\textbf{Operations (T18):} \textsc{Resolve} (Alice makes the decision by authority, with explicit conditions for revisiting).\\
\textbf{Operations (T19):} \textsc{Observe} (Bob records dissent as a public commitment).

\noindent\textbf{Why this is the most important Phase~3 moment:} This is a decision made \emph{despite} unresolved disagreement. Bob genuinely believes the team is making a mistake, and Alice acknowledges his concern but overrides it based on product priorities. The model must capture not just the decision but the \emph{structure of the disagreement}: who dissented, why, what would change their mind, and what conditions were explicitly agreed as triggers for re-evaluation.

\noindent In the formal framework, Alice's T18 is a commitment act: she commits the team to $P_5$, and the commitment is \emph{conditional}: it depends on the assumptions $a_1$--$a_3$ below. Bob's T19 is a public recording of dissent: he accepts the decision but does not retract his argument. This is different from both agreement (Bob doesn't endorse the decision) and from blocking (he doesn't prevent it).

\noindent\textbf{Symbolic model $\mathit{AF}_{19}$ (final):}
New records: $\alpha_{15} = (P_3\text{ 10--50$\times$ metadata}, b, T15, \mathit{con})$ (strengthens $\alpha_8$), $\alpha_{16} = (P_5\text{ ok for current use}, b, T17, \mathit{pro})$, $\alpha_{17} = (P_3\text{ Q2 would raise risk}, b, T17, \mathit{con})$, $\alpha_{18} = (\text{decide } P_5\text{ for launch}, a, T18, \mathit{resolve})$, $\alpha_{19} = (\text{dissent: prefers ShareDB}, b, T19, \mathit{con})$.

$\mathit{Att}_{19} = \mathit{Att}_{14} \cup \{(\alpha_{15}, P_3), (\alpha_{17}, P_3)\}$.

$\mathit{Dep}(P_5) = \{\alpha_5, \alpha_{11}, \alpha_{14}, \alpha_{16}\}$, $\mathit{Dep}(\alpha_{16}) = \{a_1, a_2\}$ and $\mathit{Dep}(\alpha_{18}) = \{P_5, a_1, a_2, a_3\}$, where $a_1 \coloneqq$ ``docs are short (5--10 pages)'', $a_2 \coloneqq$ ``editing is burst'' (Alice, T16) and $a_3 \coloneqq$ ``Q2 long-running docs not confirmed'' (Alice, T18).

\emph{Resolve:} at T18, $P_5$ is active with no recorded incoming attack and Alice holds decision authority, so the authoritative \textsc{Resolve} of \Cref{tab:ops-full} applies: $\mathit{st}(P_5) = \mathsf{resolved}$, and Alice commits the decision record $\alpha_{18}$. Bob's CRDT concerns $\alpha_8, \alpha_{10}, \alpha_{15}, \alpha_{17} \in \mathit{Cm}(b)$ are already on record (they attack $P_3$, whose CRDT basis $P_5$ shares), and T19 makes the dissent explicit as $\alpha_{19}$.

\emph{Commitments:}
$\mathit{Cm}(a) = \{\mathit{timeline}, \alpha_{18}\}$ (Alice commits to the decision).
$\mathit{Cm}(b) = \{P_1, P_2, \alpha_1, \alpha_2, \alpha_7, \alpha_8, \alpha_9, \alpha_{10}, \alpha_{12}, \alpha_{14}, \alpha_{15}, \alpha_{16}, \alpha_{17}, \alpha_{19}\}$ (Bob commits to both pro and con arguments but \textbf{not} to $\alpha_{18}$, recording dissent via $\alpha_{19}$).
$\mathit{Cm}(c) = \{P_3, P_4, P_5, \alpha_5, \alpha_6, \alpha_{11}\}$.

\emph{Decision status:} $I_1$ resolved with dissent. $P_5$ accepted. Decision authority: $a$ (product lead).
\emph{Conditional commitment:} $\mathit{Dep}(\alpha_{18}) = \{P_5, a_1, a_2, a_3\}$. If any $a_i$ is withdrawn, $\mathit{Affected}^{*}(a_i)$ reports the decision for re-evaluation.
\emph{Dissent record:} $\alpha_{19} \in \mathit{Cm}(b)$ but $\alpha_{18} \notin \mathit{Cm}(b)$. Bob accepts the decision procedurally but does not endorse it epistemically.

\paragraph{Final model.}
\begin{itemize}[leftmargin=*]
\item Issue $I_1$: Status $\to$ \textbf{decided with dissent}.
\item Decision: $P_5$ (Yjs, server-relayed via WebSocket). Decided by: Alice (product lead authority).
\item Dissent: Bob. Position: prefers ShareDB for long-term architectural safety. Objection recorded: ``short-sighted.''
\item Assumptions underlying the decision:
  \begin{itemize}[leftmargin=*, nosep]
  \item $a_1$: Documents are short (5--10 pages). \emph{Currently true.}
  \item $a_2$: Editing pattern is burst (few days of activity, then read-only). \emph{Currently true.}
  \item $a_3$: Q2 long-running project documents not confirmed. \emph{Currently true.}
  \end{itemize}
\item Conditional commitment: If Q2 roadmap confirms long-running documents, the team will re-evaluate the architecture choice.
\item Dependencies: $\mathit{Dep}(\alpha_{18}) \supseteq \{a_1, a_2, a_3\}$. Withdrawing any of them reports the decision for re-evaluation.
\item Action items: Bob writes a risk analysis with specific thresholds.
\item Positions explored and set aside (with reasons preserved):
  \begin{itemize}[leftmargin=*, nosep]
  \item $P_1, P_2$ (from-scratch OT/CRDT): set aside because of the 6-week timeline.
  \item $P_4$ (ShareDB): not chosen, since its ProseMirror integration is unclear.
  \item $P_3$ (Yjs pure P2P): set aside, since peer-to-peer sync makes access control hard, and it carries the CRDT document-size concerns.
  \end{itemize}
\item Epistemic shifts: (1)~T5: reframing from algorithm to library choice; (2)~T13: hybrid position resolves CRDT-vs-access-control tension; (3)~T18--19: decision made with recorded dissent.
\end{itemize}

\paragraph{Counterfactual test.} One week later, the Q2 roadmap is confirmed with long-running project documents. The verifier should:
\begin{enumerate}[leftmargin=*]
\item Identify that assumption $a_3$ has changed (Q2 now confirmed).
\item Trace the dependency: $a_3$ is one of the assumptions underlying the $P_5$ decision.
\item Surface the conditional commitment: the team explicitly agreed to re-evaluate in this scenario.
\item Recall Bob's document size analysis as the relevant prior concern, including his specific numbers (10--50$\times$ metadata growth) and the irreversibility argument (six-month rewrite).
\item \emph{Not} re-explore from-scratch positions ($P_1, P_2$), since they were set aside because of the timeline, which has not changed.
\item \emph{Not} re-explore pure P2P Yjs ($P_3$), since it was set aside because of access control and is not reached by $\mathit{Affected}^{*}(a_3)$.
\item Focus the re-evaluation on $P_5$ vs.\ $P_4$ (ShareDB), since the document size concern is specific to CRDTs and ShareDB was the alternative Bob advocated.
\end{enumerate}

\noindent This is the selective re-grounding the dependency map supports: an assumption change reaches exactly the records that rest on it, and questions settled for unrelated reasons stay settled.

\noindent\textbf{Symbolic model under counterfactual ($a_3$ changes):}
Assumption $a_3$ is withdrawn: after the Q2 confirmation, $\Model', w \not\models a_3$.
$\mathit{Affected}^{*}(a_3) = \{\alpha_{18}\}$: the decision is the only record that rests on $a_3$ (\Cref{def:affected}). It is reported for re-evaluation without changing its standing.
Bob's Q2 concern $\alpha_{17}$ and the storage concerns $\alpha_8, \alpha_{15}$ are the recorded starting points for the re-evaluation, and his dissent $\alpha_{19}$ names the alternative $P_4$ (ShareDB).
Arguments $\alpha_1, \alpha_2$ ($P_1, P_2$ too risky) are \emph{not} affected: $\mathit{Dep}(\alpha_1) = \mathit{Dep}(\alpha_2) = \{\mathit{timeline}\}$, and the timeline has not changed.
$\alpha_{12}$ ($P_3$ P2P access control) is \emph{not} affected, since no recorded dependency path leads from it to $a_3$.

\FloatBarrier
\section{Experimental Details}
\label{app:setup}

\FloatBarrier
\subsection{Benchmarks and Baselines}

\noindent\textbf{ReviseQA.} We use all $930$ verified scenarios of the public release ($6{,}510$ edit steps, extending the $210$-example subset of the original paper) under the explicit-edit track, with no correction feedback or demonstration reasoning. Accuracy@$k$ has easy, medium and hard levels at $k=2$, $4$ and $7$. We report every depth. RAG retrieves top-$22$ statements from the current post-edit context using TF-IDF. The budget matches the median of $22$ recorded premises across the $6{,}510$ edit steps. We add RAG as a retrieval baseline alongside the benchmark's LLM-only protocol. The end-to-end arm uses a single GPT-4o interpretation of each scenario, shared by all QA models. Per-depth results are in \Cref{app:reviseqa-models}.

\noindent\textbf{MemAB-FC.} The FactConsolidation split of MemoryAgentBench \citep{hu2025memoryagentbench} streams $455$ to $18{,}332$ numbered facts at 6K, 32K, 64K and 262K tokens, with $35$--$40\%$ superseded under the larger-serial-wins rule. Each length has $100$ single-hop and $100$ multi-hop questions. LLM-only reads the stream tail-truncated to the QA model's window. RAG retrieves top-$50$ individual facts from the full stream using TF-IDF. Verifier (native) maintains standing and selects active records by following subject--object links from entities named in the question. A separate control applies the same TF-IDF retriever at top-$50$ to the active records, isolating the effect of standing under a fixed retrieval configuration. The MemAB-FC protocol and per-length results are in \Cref{app:cr}.

\noindent\textbf{BEAM.} The 100K split of BEAM: $20$ conversations, $400$ probing questions, ten memory abilities, scored against the benchmark's rubric criteria with its judge prompt and Qwen2.5-32B-Instruct-AWQ as judge. RAG and the verifier both rank with BEAM's retriever, bge-small-en-v1.5, and the RAG arm is the benchmark's own design, retrieving BEAM's user--assistant turn pairs up to the verifier's budget. Ingestion, arm budgets and the revised/stands split are in \Cref{app:beam}.

\noindent\textbf{Authored scenarios and RECON.} The Phase~2 and Phase~3 scenarios of \Cref{sec:phases} support per-utterance classification (\Cref{tab:cross-model,tab:phase3-results}). Phase~2 and RECON's released provenance skeletons \citep{arya2026recon} supply ground-truth dependency structure for the soundness check (\Cref{app:recon}).

\FloatBarrier
\subsection{Models, Prompt Conditions and Serving}

\noindent\textbf{Interpreter models and prompt conditions.} The classification experiment runs Claude Sonnet 4, GPT-4o, Qwen2.5-7B and Llama-3.1-8B (the last two via vLLM) under three prompt conditions, \emph{Minimal}, \emph{Definitions} and \emph{State-augmented}, with $5$ runs per cell, the first at temperature $0$ and the other four at $1.0$. \emph{Minimal} gives the eight operation names with a one-line gloss each. \emph{Definitions} adds a full definition and a generic example per operation, plus explicit rules for the neighbouring pairs (\textsc{Undermine} vs.\ \textsc{Revise}, \textsc{Expand-Awareness} vs.\ \textsc{Observe}, when \textsc{Hypothesize} and \textsc{Revise} co-occur). \emph{State-augmented} further appends the engine state rendered before six checkpoint turns: which hypotheses are active, weakened or abandoned, and when awareness has expanded. Phase~3 uses Definitions only. Scoring is multi-label F1 and exact match per utterance, with separate reporting on the four \emph{key shifts} per scenario, the turns at which standing or awareness changes, drawn from awareness expansion, undermining, hypothesis abandonment, causal reversal and decision with dissent.

\noindent\textbf{Model identifiers and serving.} Closed models are called by API alias. Open-weight models are served with vLLM~0.19 on a single A100-40GB.
\begin{center}\small
\begin{tabular}{@{}ll@{}}
\toprule
Name in the paper & Identifier \\
\midrule
GPT-4o & \texttt{gpt-4o} \\
GPT-4o-mini & \texttt{gpt-4o-mini} \\
Claude Sonnet 4 & \texttt{claude-sonnet-4-20250514} \\
Qwen2.5-7B & \texttt{Qwen/Qwen2.5-7B-Instruct} \\
Qwen2.5-14B & \texttt{Qwen/Qwen2.5-14B-Instruct} \\
Gemma-3-12B & \texttt{unsloth/gemma-3-12b-it} \\
Llama-3.1-8B & \texttt{meta-llama/Llama-3.1-8B-Instruct} \\
BEAM judge & \texttt{Qwen/Qwen2.5-32B-Instruct-AWQ} \\
BEAM retriever & \texttt{BAAI/bge-small-en-v1.5} \\
\bottomrule
\end{tabular}
\end{center}

\FloatBarrier
\subsection{Statistics and Supersession Capture}
\label{app:capture}

\noindent\textbf{Statistics.} Paired comparisons use exact McNemar tests \citep{mcnemar1947note} on discordant pairs: scenarios on ReviseQA and questions on MemAB-FC, where questions within one variant share a single stream and the four stream lengths are independent streams. On BEAM the scores are graded, so paired comparisons are Wilcoxon signed-rank tests \citep{wilcoxon1945individual} over questions, per QA model and pooled over the five with each (model, question) pair as one observation. The intervals in \Cref{fig:capability}c are $95\%$ percentile bootstraps over questions.

\noindent\textbf{Supersession capture and state agreement.} The measurements in \Cref{sec:analysis} come from the benchmark annotations and the end-to-end arm's update logs. On ReviseQA, $4{,}361$ of the $6{,}510$ edit steps ($67.0\%$) remove a premise inside the conclusion's maintained dependency set, and the GPT-4o Interpreter emits a removal on $4{,}355$ of them ($99.9\%$). On MemAB-FC, $35.4$--$39.5\%$ of the streamed facts are superseded under the larger-serial-wins rule ($161/455$ at 6K up to $7{,}237/18{,}332$ at 262K), and the dependency map the Interpreter builds matches the one built from the structured updates at Jaccard $1.000$ at 6K and $0.989$ at 32K.

%==================================================================
\FloatBarrier
\section{Full Benchmark Results}
\label{app:benchmarks}
%==================================================================

\FloatBarrier
\subsection{ReviseQA: Full Per-Depth Results}
\label{app:reviseqa-models}

\begin{figure}[!htbp]
\centering
\includegraphics[width=\linewidth]{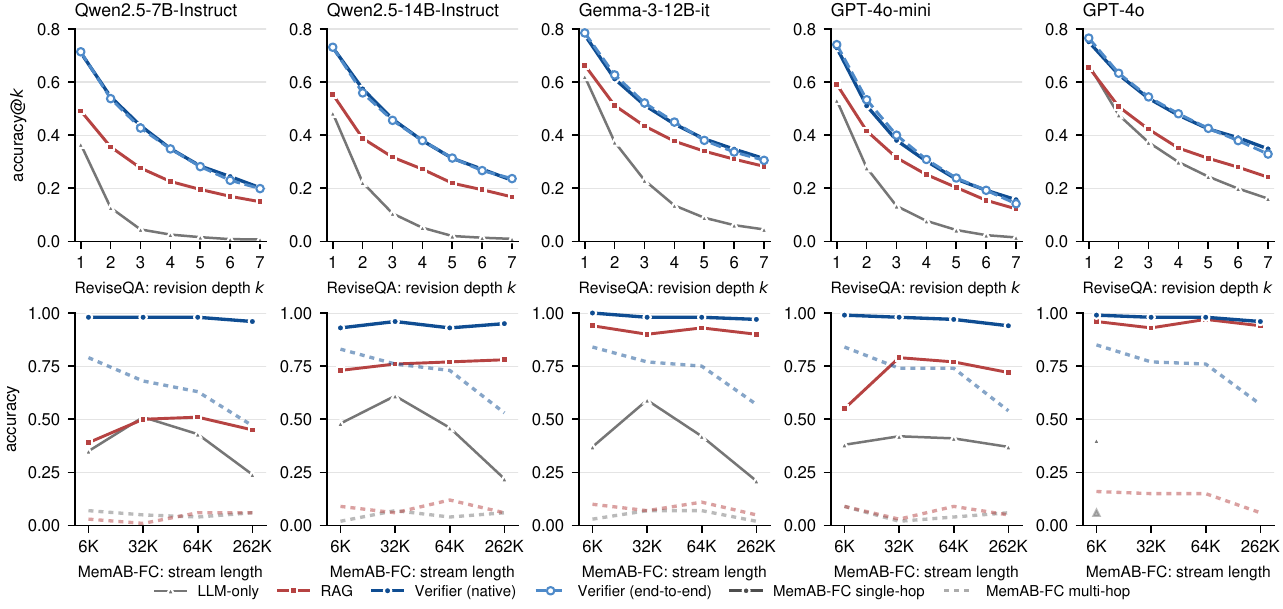}
\caption{One column per QA model. Top: ReviseQA accuracy@$k$ under all four arms. Bottom: MemAB-FC single-hop (solid) and multi-hop (dotted) accuracy across stream lengths for LLM-only, RAG and Verifier (native).}
\label{fig:reviseqa-lcata}
\end{figure}

\Cref{fig:reviseqa-lcata} plots accuracy@$k$ for all five QA models and all four arms, with RAG at top-$22$. The native arm's lead over RAG is significant at every depth on four QA models and through $k\le6$ on Gemma-3-12B (exact McNemar, $p<0.05$). At the benchmark's medium cut-off, $k{=}4$, the native arm scores $0.302$--$0.475$ against RAG's $0.226$--$0.377$ and LLM-only's $0.026$--$0.299$. The end-to-end arm replaces the benchmark's edit annotations with a GPT-4o Interpreter that translates each natural-language edit into engine updates. Its accuracy tracks the native arm at every depth on all five models, differing by at most $2.26$ percentage points at $k{=}2$, $0.75$ at $k{=}4$ and $1.94$ at $k{=}7$. The end-to-end arm stays above RAG at every depth on all five models.

\subsection{MemAB-FC: Protocol and Per-Length Results}
\label{app:cr}

\noindent\textbf{Protocol.} The benchmark supplies serially numbered natural-language facts and instructs agents that newer facts carry larger serial numbers. We parse each fact with $37$ hand-written relation templates and a last-``is'' fallback. Two facts share a supersession key when they share (template, subject). Under the larger-serial-wins rule, the latest record for each key remains active. A later fact assigning a new value to the same attribute of the same subject is treated as an observation of the new value and a revision of the superseded record. All arms use the same QA prompt and normalised-containment scoring against the gold answer list. LLM-only reads the raw stream tail-truncated to the QA model's context budget. GPT-4o LLM-only is reported at 6K. Each variant contains $100$ questions sharing one stream, and the four stream lengths use independent streams.

\noindent\textbf{Standing control.} RAG and Verifier + TF-IDF use the same deterministic TF-IDF retriever, top-$50$ budget and QA configuration. RAG retrieves from the full stream, whereas Verifier + TF-IDF retrieves only from records in good standing. Mean prompt lengths differ by less than $3\%$ in every variant. This comparison isolates the effect of maintaining standing under a fixed retrieval configuration. The control leads RAG at every length on single-hop and at every length on multi-hop except a tie for Qwen2.5-7B at 262K.

\noindent\textbf{Record selection.} Verifier (native), the arm in the main results, uses the same active records but selects them by following subject--object links from entities named in the question. Selection starts from at most eight matching subjects and follows links for up to three steps, with at most eight frontier entities per step and $50$ selected records. The engine registers the selected record IDs as the dependencies of a question-specific hypothesis before the QA call. The QA model reads the selected facts in serial order. This arm leads RAG at every length on both hop types for all five QA models (\Cref{tab:cr-perlength}). Pooled over the four lengths, both Verifier and Verifier + TF-IDF lead RAG significantly for every model and hop type (exact McNemar, $p<0.05$).

\begin{table}[!htbp]
\centering
\caption{MemAB-FC accuracy per stream length ($n{=}100$ questions per cell). Verifier uses the record selection of the main results. Verifier + TF-IDF is the standing control, retrieving top-$50$ active records with the same retriever as RAG. Both maintained-state arms use structured updates. Bold marks the best accuracy within each model and column.}
\label{tab:cr-perlength}
\small
\resizebox{\ifdim\width>\textwidth\textwidth\else\width\fi}{!}{
\begin{tabular}{llcccccccc}
\toprule
& & \multicolumn{4}{c}{Single-hop} & \multicolumn{4}{c}{Multi-hop} \\
\cmidrule(lr){3-6}\cmidrule(lr){7-10}
QA model & Arm & 6K & 32K & 64K & 262K & 6K & 32K & 64K & 262K \\
\midrule
Qwen2.5-7B & LLM-only & 0.35 & 0.51 & 0.43 & 0.24 & 0.07 & 0.05 & 0.04 & 0.06 \\
 & RAG & 0.39 & 0.50 & 0.51 & 0.45 & 0.03 & 0.01 & 0.06 & 0.06 \\
 & Verifier + TF-IDF & 0.95 & 0.94 & 0.92 & 0.91 & 0.18 & 0.10 & 0.14 & 0.06 \\
 & Verifier & \textbf{0.98} & \textbf{0.98} & \textbf{0.98} & \textbf{0.96} & \textbf{0.79} & \textbf{0.68} & \textbf{0.63} & \textbf{0.47} \\
\midrule
Qwen2.5-14B & LLM-only & 0.48 & 0.61 & 0.46 & 0.22 & 0.02 & 0.07 & 0.04 & 0.06 \\
 & RAG & 0.73 & 0.76 & 0.77 & 0.78 & 0.09 & 0.06 & 0.12 & 0.06 \\
 & Verifier + TF-IDF & \textbf{0.98} & \textbf{0.97} & \textbf{0.98} & 0.94 & 0.36 & 0.14 & 0.20 & 0.10 \\
 & Verifier & 0.93 & 0.96 & 0.93 & \textbf{0.95} & \textbf{0.83} & \textbf{0.76} & \textbf{0.73} & \textbf{0.53} \\
\midrule
Gemma-3-12B & LLM-only & 0.37 & 0.59 & 0.42 & 0.21 & 0.03 & 0.07 & 0.07 & 0.02 \\
 & RAG & 0.94 & 0.90 & 0.93 & 0.90 & 0.10 & 0.07 & 0.11 & 0.05 \\
 & Verifier + TF-IDF & \textbf{1.00} & \textbf{0.98} & \textbf{0.98} & \textbf{0.97} & 0.39 & 0.21 & 0.26 & 0.15 \\
 & Verifier & \textbf{1.00} & \textbf{0.98} & \textbf{0.98} & \textbf{0.97} & \textbf{0.84} & \textbf{0.77} & \textbf{0.75} & \textbf{0.57} \\
\midrule
GPT-4o-mini & LLM-only & 0.38 & 0.42 & 0.41 & 0.37 & 0.09 & 0.02 & 0.04 & 0.06 \\
 & RAG & 0.55 & 0.79 & 0.77 & 0.72 & 0.09 & 0.03 & 0.09 & 0.05 \\
 & Verifier + TF-IDF & \textbf{0.99} & \textbf{0.98} & \textbf{0.98} & \textbf{0.97} & 0.38 & 0.20 & 0.22 & 0.11 \\
 & Verifier & \textbf{0.99} & \textbf{0.98} & 0.97 & 0.94 & \textbf{0.84} & \textbf{0.74} & \textbf{0.74} & \textbf{0.54} \\
\midrule
GPT-4o & LLM-only & 0.40 & --- & --- & --- & 0.06 & --- & --- & --- \\
 & RAG & 0.96 & 0.93 & 0.97 & 0.94 & 0.16 & 0.15 & 0.15 & 0.06 \\
 & Verifier + TF-IDF & \textbf{0.99} & \textbf{0.98} & \textbf{0.98} & 0.95 & 0.48 & 0.27 & 0.32 & 0.15 \\
 & Verifier & \textbf{0.99} & \textbf{0.98} & \textbf{0.98} & \textbf{0.96} & \textbf{0.85} & \textbf{0.77} & \textbf{0.76} & \textbf{0.57} \\
\bottomrule
\end{tabular}}
\end{table}

\noindent\textbf{End-to-end ingestion.} A GPT-4o Interpreter ingested the 6K and 32K streams, deciding which of the five most similar active facts, if any, each incoming fact supersedes, without hand-written supersession templates. The resulting dependency map matches the one built from the structured updates at Jaccard $1.000$ at 6K and $0.989$ at 32K. We evaluate both selection procedures over this map. For the TF-IDF control, question-level correctness agrees with the structured-update run on $392$--$398$ of $400$ questions across the five QA models. For the Verifier (native) selection, the three evaluated models differ from their native counterparts by at most $1.5$ percentage points on either hop type when pooling the two lengths, and remain above RAG (\Cref{tab:cr-e2e-selection}).

\begin{table}[!htbp]
\centering
\caption{MemAB-FC with the Verifier (native) selection: native and end-to-end accuracy pooled over 6K and 32K ($n{=}200$ questions per hop type and model).}
\label{tab:cr-e2e-selection}
\small
\begin{tabular}{lcccc}
\toprule
& \multicolumn{2}{c}{Single-hop} & \multicolumn{2}{c}{Multi-hop} \\
\cmidrule(lr){2-3}\cmidrule(lr){4-5}
QA model & Native & End-to-end & Native & End-to-end \\
\midrule
Qwen2.5-7B & 0.980 & 0.975 & 0.735 & 0.740 \\
Qwen2.5-14B & 0.945 & 0.940 & 0.795 & 0.800 \\
Gemma-3-12B & 0.990 & 0.985 & 0.805 & 0.820 \\
\bottomrule
\end{tabular}
\end{table}

\noindent\textbf{Published systems.} MemoryAgentBench reports FactConsolidation results with a GPT-4o-mini backbone. With this backbone, our LLM-only arm scores $0.395/0.052$ single-/multi-hop, compared with the published long-context agent's $0.45/0.05$. Our fact-level RAG scores $0.708/0.065$, compared with $0.48/0.03$ for the benchmark's BM25 agent over 512-token chunks. Verifier + TF-IDF reaches $0.980/0.228$ under the same retrieval configuration as our RAG, and Verifier (native) reaches $0.970/0.715$.

\subsection{MemAB-FC: Multi-hop Analysis}
\label{app:mh-errors}

\noindent\textbf{Chain analysis.} We run a chain search for each of the $400$ multi-hop questions. Starting from the gold answer, we search backwards through records whose object matches the current target, then continue from their subjects, for at most six links. A chain is found when a subject occurs in the question. The search finds a chain over active records for $336$ questions and none for the remaining $64$. These labels describe subject--object connectivity to the gold answer. We then apply the same search to the records selected by each arm.

\noindent\textbf{Evidence selection.} We record the evidence each arm passes to the QA model. Fact and active-record counts per length are $455/294$, $2{,}310/1{,}473$, $4{,}580/2{,}889$ and $18{,}332/11{,}095$. Chain labels agree across the five QA models within each arm.

\begin{table}[!htbp]
\centering
\caption{MemAB-FC multi-hop chain coverage and QA accuracy under structured updates. Each question contributes five question--model pairs to the pooled accuracy. The two arms use the same active state and chain diagnostic.}
\label{tab:mh-errors}
\small
\resizebox{\ifdim\width>\textwidth\textwidth\else\width\fi}{!}{
\begin{tabular}{lcccc}
\toprule
& \multicolumn{2}{c}{Verifier + TF-IDF} & \multicolumn{2}{c}{Verifier} \\
\cmidrule(lr){2-3}\cmidrule(lr){4-5}
Chain diagnostic & Questions & Accuracy & Questions & Accuracy \\
\midrule
Chain contained in selected records & 107 & 0.636 & 291 & 0.928 \\
Active chain found, not fully selected & 229 & 0.071 & 45 & 0.102 \\
No active chain found & 64 & 0.066 & 64 & 0.134 \\
\bottomrule
\end{tabular}}
\end{table}

Among the $1{,}558$ incorrect question--model pairs under the TF-IDF control, $1{,}064$ have an active chain that is not fully contained in the retrieved facts. The Verifier (native) selection raises the number of questions with a chain in the QA input from $107$ to $291$ (\Cref{tab:mh-errors}). On these groups, pooled accuracy is $0.636$ for the TF-IDF control and $0.928$ for Verifier. These groups contain different questions. For Verifier, $1{,}350$ of the $1{,}455$ pairs with a selected chain are answered correctly.

\noindent\textbf{Figure~\ref{fig:motivation} example.} The example is the 18th question in the MemAB-FC 6K multi-hop split, with Qwen2.5-7B for QA. RAG retrieves facts $29$, $123$ and $163$ but not $33$, and answers Paris. Verifier (native) selects records $33$, $113$ and $123$, including the two records linking Madame du Barry to London, and answers London. With GPT-4o interpretation, the dependency map contains $294$ active records from $455$ input facts, and Verifier (end-to-end) also answers London.

\noindent\textbf{Superseded gold answers.} For $61$ of the $64$ questions with no active chain found, the search finds a chain when superseded records are included. For example, one question asks for the capital of the country that gave rise to the sport played by Rog\'{e}rio Ceni. A chain using superseded records passes through \emph{rugby union was created in the country of India} ($\#186$) and \emph{The capital of India is New Delhi} ($\#158$). Fact $\#341$ replaces the latter with Grosseto, while the gold answer remains New Delhi.

\FloatBarrier
\FloatBarrier
\subsection{BEAM: Protocol, Split and Results}
\label{app:beam}

\noindent\textbf{Data.} We use the whole 100K split of BEAM \citep{tavakoli2026beam}: $400$ probing questions with $1{,}051$ rubric criteria. Every answer is scored with the benchmark's unified judge prompt, with Qwen2.5-32B-Instruct-AWQ as judge, one call per criterion with scores in $\{0, 0.5, 1\}$, averaged per question. BEAM scores event ordering with a Kendall $\tau_b$ coefficient over the order and presence of events. We score it with the same rubric judge as the other nine abilities so that all ten share one scale. All arms answer with the benchmark's own RAG answer prompt and differ only in the context.

\noindent\textbf{Ingestion.} A GPT-4o Interpreter extracts at most $20$ atomic statements per message from both speakers, $53{,}222$ in all. Each statement is decided against the live active set: the candidates are the most similar active statements under three retrieval views, capped at $20$, and one call returns which candidate, if any, states an earlier value of the same attribute and whether the value changed. The ownership rule keeps a user commitment active when the revising statement is the assistant's. In all, $18{,}696$ statements are retired. The dependency map is built once and shared by every QA model.

\noindent\textbf{Arms and budgets.} \emph{LLM-only} reads the transcript tail-truncated to the QA model's window, $12.0$K tokens for the 16K-window open models and $109.6$K tokens for GPT-4o and GPT-4o-mini. Eighteen of the $20$ conversations exceed the latter. \emph{RAG} follows BEAM's own design, bge-small-en-v1.5 over user--assistant turn pairs, with pairs added in similarity order up to the length of the verifier's top-$100$ active statements for that question, $3.3$K tokens on average. \emph{Verifier (end-to-end)} ranks the active statements with the same retriever, adds the source turn of each in rank order up to the same length, and appends a retraction list of every statement in those turns that was later replaced by a different value, together with its replacement, $3.7$K tokens in all on average. A turn whose statements are all retired never enters.

\noindent\textbf{Reading the split.} For each question BEAM lists the evidence turns its answer rests on. A question is labelled \emph{revised} when one of those turns holds a statement that the dependency map marks as replaced by a later statement with a different value, and \emph{stands} otherwise. The label is read from the map before any answer is scored. Abstention questions have no evidence turns, nor do four summarization and one preference question, so $355$ of the $400$ carry a label, $156$ stand and $199$ revised, across nine abilities. RAG and the end-to-end verifier were run three times with identical settings for every QA model but GPT-4o, and \Cref{tab:beam-all} and \Cref{fig:capability}c report the mean of the three runs. A repeated run changes a model's score by at most $0.010$. Pooled over the five QA models, the end-to-end verifier scores $0.347$ against RAG's $0.353$ on the $780$ question--model pairs where every premise still stands and $0.249$ against $0.202$ on the $995$ where one was revised. Paired Wilcoxon tests put the verifier ahead where a premise was revised on three QA models ($p{<}0.05$) and find no difference where every premise still stands on any of the five.

\begin{table}[!htbp]
\centering
\caption{BEAM, all $400$ probing questions per QA model: score over all ten abilities\,/\,on knowledge update. RAG and Verifier: mean of three runs (GPT-4o: one run).}
\label{tab:beam-all}
\small
\resizebox{\ifdim\width>0.86\textwidth 0.86\textwidth\else\width\fi}{!}{
\begin{tabular}{lccc}
\toprule
QA model & LLM-only & RAG & Verifier \\
\midrule
Qwen2.5-7B           & 0.166\,/\,0.087 & 0.256\,/\,0.258 & \textbf{0.301}\,/\,\textbf{0.350} \\
Qwen2.5-14B          & 0.198\,/\,0.062 & 0.340\,/\,0.400 & \textbf{0.341}\,/\,\textbf{0.433} \\
Gemma-3-12B          & 0.136\,/\,0.075 & 0.258\,/\,\textbf{0.263} & \textbf{0.259}\,/\,0.258 \\
GPT-4o-mini          & 0.279\,/\,0.300 & 0.350\,/\,0.425 & \textbf{0.351}\,/\,\textbf{0.446} \\
GPT-4o               & 0.286\,/\,0.325 & 0.328\,/\,0.400 & \textbf{0.362}\,/\,\textbf{0.487} \\
\midrule
Pooled ($n{=}2000$\,/\,$200$) & 0.213\,/\,0.170 & 0.306\,/\,0.349 & \textbf{0.323}\,/\,\textbf{0.395} \\
\bottomrule
\end{tabular}
}
\end{table}

\noindent\textbf{Per-model results.} Averaged over all ten abilities the end-to-end verifier is level with or ahead of RAG for every QA model (\Cref{tab:beam-all}).

%==================================================================
\FloatBarrier
\section{Analysis Experiments}
\label{app:analysis-exp}
%==================================================================

\FloatBarrier
\subsection{Soundness of $\mathit{Affected}^{*}$: Phase 2 and RECON}
\label{app:recon}

\Cref{tab:recon} collects every measurement of the structural check given ground-truth dependency structure: the authored Phase~2 scenario and the third-party RECON benchmark, under the one-step and iterated queries, with GPT-4o answering the same question as a reference. RECON \citep{arya2026recon} is a benchmark of $24$ investigation case files ($50$--$100$K tokens each) with six memory tasks. Its \emph{counterfactual} task queries $\mathit{Affected}^{*}(\beta)$ for a trigger record $\beta$: each case ships a provenance skeleton whose evidence chains carry explicit \texttt{requires}/\texttt{establishes} attributes, and each counterfactual scenario names a trigger step together with the steps that depend on it and those that do not. We verified that the gold \texttt{dependent\_steps} equal within-chain reachability over those attributes for all $64$ scenarios of the $16$ medical and finance case files ($32$ per domain). The medical and finance skeletons serialise these attributes, and the dependents in the counterfactual task span several chain links, so the task exercises the walk.

\noindent\textbf{Oracle protocol.} Each chain is loaded into the engine as hypotheses with $\mathit{Dep}$ set from \texttt{requires}/\texttt{establishes}. We then query the trigger. Two LLM arms answer the same question with GPT-4o at temperature $0$: \emph{isolated chain} receives only the trigger's chain (${\sim}550$ tokens), and \emph{full skeleton} receives the entire case skeleton (${\sim}51$K tokens), RECON's own oracle setting. Scoring counts every predicted step outside the gold dependent set as a false positive. Phase~2 contributes $16$ decisions ($9$ dependent) from $4$ retraction queries over $4$ hypotheses. The one-step query decides $12$ of the $16$ correctly and $\mathit{Affected}^{*}$ all $16$. RECON contributes $64$ scenarios with $198$ dependent steps. Each of the full-skeleton reader's three inexact answers includes one extra step.

\begin{table}[!htbp]
\centering
\caption{Soundness of the retraction query on Phase~2 and RECON given ground-truth dependency structure. Exact: whole affected set correct; TP/FP/FN over dependent steps, any step outside the gold set counted as FP.}
\label{tab:recon}
\small
\begin{tabular}{llccc}
\toprule
Query & Input & Exact & TP/FP/FN & Cost per query \\
\midrule
\multicolumn{5}{l}{\emph{Phase~2 (authored dialogue, \Cref{app:phase2})}} \\
\quad one-step $\mathit{Affected}$ (\Cref{def:affected}) & GT $\mathit{Dep}$ & $1/4$ & $5/0/4$ & $<1\,\mu$s \\
\quad $\mathit{Affected}^{*}$ (\Cref{cor:dep-closure}) & GT $\mathit{Dep}$ & $\mathbf{4/4}$ & $9/0/0$ & $<1\,\mu$s \\
\midrule
\multicolumn{5}{l}{\emph{RECON (third-party, 16 case files)}} \\
\quad one-step $\mathit{Affected}$ & skeleton & $13/64$ & $66/0/132$ & $<1\,\mu$s \\
\quad $\mathit{Affected}^{*}$ & skeleton & $\mathbf{64/64}$ & $198/0/0$ & $0.4\,\mu$s \\
\quad GPT-4o, isolated chain & ${\sim}550$ tokens & $64/64$ & $198/0/0$ & 1 call \\
\quad GPT-4o, full skeleton & ${\sim}51$K tokens & $61/64$ & $198/3/0$ & 1 call, ${\sim}\$0.13$ \\
\bottomrule
\end{tabular}
\end{table}

\noindent On RECON's own multiple-choice and count questions for these scenarios, $103$ of $186$ are answerable from the affected set, with options matched to chain steps by content-word overlap. On these, $\mathit{Affected}^{*}$ scores $0.948$ under RECON's penalised scoring, against $0.615$ for one-step $\mathit{Affected}$. 

\noindent\textbf{Reading.} The one-step query misses exactly the multi-hop cases, which is why the verifier runs the iterated query $\mathit{Affected}^{*}$ (\Cref{cor:dep-closure}). $\mathit{Affected}^{*}$ matches the gold on every Phase~2 decision and every RECON scenario, at a per-query cost that remains nearly constant over the measured state sizes (\Cref{fig:cost}c).

\FloatBarrier
\subsection{Retraction Latency and Token Cost}
\label{app:e4}

\noindent\textbf{Latency protocol.} We measure the query on synthetic dependency maps $\mathcal{D}_t$ at $K \in \{13, \ldots, 2000\}$ turns under two regimes: \emph{naive}, with $|\mathit{Args}_t| \sim 0.3K$, and \emph{bounded}, with $|\mathit{Args}_t| \leq 50$, modelling Phase-2-like working sets where resolved hypotheses retire. In both regimes the verifier stays at $0.25$--$0.46\,\mu$s per query while history replay grows linearly from $1.1\,\mu$s to $190\,\mu$s. \Cref{fig:cost}c reports medians over $5$ seeds $\times$ $200$ queries per seed in the naive regime, with shading spanning the per-seed medians. The bounded regime coincides with it. Per-query cost varies little with $K$ and $|\mathit{Args}_t|$ and is governed by the size of the affected set, whose mean is below $1$ in these states. \Cref{prop:complexity}'s $\bigO(|\mathit{Args}| + |\mathit{Dep}|)$ is the dense-graph worst case.

\noindent\textbf{Token accounting.} For the per-query prompt costs in \Cref{fig:cost}a,b, we count the tokens of the exact prompt each arm of \Cref{sec:experiment} sends for each MemAB-FC variant and question, using the GPT-4o tokenizer. Mean prompt lengths across the eight variants are $753$--$811$ tokens for RAG and $97$--$174$ for the native verifier, while the full-context prompt for the GPT models grows from $6.0$K to $114.5$K, tail-truncated to their input budget at $262$K. For the open models it is truncated from $32$K onward, at $11.4$--$11.7$K. Untruncated, the $262$K stream costs $268$K tokens per query. The active-record TF-IDF control and RAG both retrieve top-$50$ facts, with mean prompt lengths within $3\%$ of each other at every variant. On the write side, the end-to-end arm's Interpreter ingestion costs $96.6$K tokens for the $6$K stream and $494.7$K for the $32$K stream. Using the native verifier's prompt lengths to estimate savings against untruncated full-context QA gives token-denominated break-even at $16.3$ queries ($6$K) and $15.3$ queries ($32$K), of the $100$ queries per variant. Verifier (native) ingests the benchmark's numbered facts by template and pays no LLM tokens at write time.

%==================================================================
\FloatBarrier
\section{Comparison with Related Work}
\label{app:comparison}
%==================================================================

\Cref{tab:comparison} compares five capabilities. \emph{Incremental updates} maintain a persistent state as interactions arrive. \emph{Fact revision} updates or invalidates stored facts superseded by new information. \emph{Support links} explicitly connect conclusions to supporting premises or evidence. \emph{Retraction effects} compute the transitive dependents of a withdrawn premise. \emph{Runtime checks} test whether a candidate continuation's recorded support chain remains in good standing at the current turn.

\begin{table}[!htbp]
\centering
\caption{Capabilities of conversational memory and reasoning systems. A checkmark denotes a mechanism specified in the cited work.}
\label{tab:comparison}
\small
\begin{tabular*}{\textwidth}{@{\extracolsep{\fill}}lccccc@{}}
\toprule
System & \shortstack{Incremental\\updates} & \shortstack{Fact\\revision} & \shortstack{Support\\links} & \shortstack{Retraction\\effects} & \shortstack{Runtime\\checks} \\
\midrule
Mem0 & \checkmark & \checkmark & & & \\
A-MEM & \checkmark & & & & \\
Generative Agents & \checkmark & & \checkmark & & \\
Zep/Graphiti & \checkmark & \checkmark & & & \\
D-SMART & \checkmark & \checkmark & & & \\
VISTA & & & \checkmark & & \\
\midrule
\textbf{Ours} & $\boldsymbol{\checkmark}$ & $\boldsymbol{\checkmark}$ & $\boldsymbol{\checkmark}$ & $\boldsymbol{\checkmark}$ & $\boldsymbol{\checkmark}$ \\
\bottomrule
\end{tabular*}
\end{table}

\noindent\textbf{Memory updates and support links.} Mem0 updates or deletes stored facts \citep{chhikara2025mem0}. Zep invalidates conflicting fact edges while retaining history \citep{rasmussen2025zep}. D-SMART prunes conflicting triples and searches reasoning trajectories over its knowledge graph \citep{lei2025dsmart}. A-MEM updates note context, keywords, and tags and creates associative links \citep{xu2025amem}. It does not specify a mechanism for retiring superseded facts or recording conclusion-to-premise support. Generative Agents store reflection-to-evidence pointers \citep{park2023generative}, and VISTA reconstructs reasoning dependency trees from conversation histories \citep{zhang2025vista}. Our verifier uses standing and support links together to check candidate continuations and identify exactly the recorded transitive dependents of a withdrawn premise, preserving unaffected claims (\Cref{cor:dep-closure}).

\noindent\textbf{Related verification objectives.} Dynamic epistemic friction uses DEL to model resistance to belief updates in dialogue \citep{obiso2025friction}. AgentArmor constructs program dependency graphs from tool-execution traces and enforces security policies through a type system \citep{wang2025agentarmor}. Neither procedure checks the current standing of a conversational conclusion's support chain. Our verifier performs this check over the maintained claim dependencies, enabling detection of reliance on withdrawn premises even when the proposed continuation violates no tool-use policy.

\end{document}